\documentclass[nohyperref]{article}

\usepackage{microtype}
\usepackage{subcaption}
\usepackage{graphicx}
\usepackage{lipsum}
\usepackage{booktabs} 
\usepackage{nicematrix}
\usepackage{array,multirow,graphicx}

\usepackage{hyperref}
\usepackage{xurl}

\usepackage[accepted]{icml2022}

\usepackage{amsmath}
\usepackage{amssymb}
\usepackage{mathtools}
\usepackage{amsthm}
\usepackage{dsfont}
\usepackage{xcolor}   
\usepackage{graphicx}

\usepackage{standalone}
\usepackage{fancyhdr}
\usepackage{algorithm}
\usepackage{algorithmic}

\usepackage{forloop}
\usepackage{times}
\usepackage{wrapfig}
\usepackage{siunitx}
\usepackage{url} 
\usepackage{bm}
\usepackage{nicefrac}

\newcommand{\punt}[1]{}

\usepackage{amsthm}
\usepackage{xcolor}
\usepackage{hyperref}

\newcommand{\trp}{{^\top}} 

\renewcommand{\eqref}[1]{eq.~\ref{eq:#1}}
\newcommand{\Nrm}{\mathcal{N}}

\newcommand{\figref}[1]{Fig.~\ref{fig:#1}}  
\newcommand{\secref}[1]{Sec.~\ref{sec:#1}}
\newcommand{\tabref}[1]{Table ~\ref{tab:#1}}  

\newcommand{\suppsecref}[1]{Appendix Sec.~\ref{supp:#1}}  

\newcommand{\vx}{\mathbf{x}}

\newcommand{\vl}{\mathbf{l}}

\newcommand{\Dat}{\mathcal{D}}

\newcommand{\vf}{\mathbf{f}}

\newcommand{\vphi}{\mathbf{\ensuremath{\bm{\phi}}}}

\newcommand{\vk}{\mathbf{k}}

\newcommand{\ve}{\mathbf{e}}

\newcommand{\vmu}{\mathbf{\ensuremath{\bm{\mu}}}}
\newcommand{\vtheta}{\mathbf{\ensuremath{\bm{\theta}}}}

\newcommand{\mI}{\mathbf{I}}

\newcommand{\vy}{\mathbf{y}}

\newcommand{\vg}{\mathbf{g}}

\newtheorem{thm}{Theorem}[section]
\newtheorem{lem}{Lemma}[section]

\newtheorem{prop}[thm]{Proposition}
\newtheorem{remark}{Remark}

\def\Pr{\ensuremath{\text{Pr}}}

\newcommand{\vh}{\mathbf{h}}

\usepackage[textsize=tiny]{todonotes}

\icmltitlerunning{Hermite Polynomial Features for Private Data Generation}

\begin{document}

\twocolumn[

\icmltitle{Hermite Polynomial Features for Private Data Generation}



\icmlsetsymbol{equal}{*}

\begin{icmlauthorlist}
\icmlauthor{Margarita Vinaroz}{equal,mpi}
\icmlauthor{Mohammad-Amin Charusaie}{equal,mpi}
\icmlauthor{Frederik Harder}{mpi}
\icmlauthor{Kamil Adamczewski}{mpi}
\icmlauthor{Mi Jung Park}{ubc}
\end{icmlauthorlist}

\icmlaffiliation{mpi}{Max Planck Institute for Intelligent Systems, Tuebingen, Germany}
\icmlaffiliation{ubc}{University of British Columbia, Vancouver, Canada. CIFAR AI Chair at AMII}

\icmlcorrespondingauthor{Mi Jung Park}{mijungp@cs.ubc.ca}

\icmlkeywords{Machine Learning, ICML}

\vskip 0.3in
]



\printAffiliationsAndNotice{\icmlEqualContribution} 

\begin{abstract}

Kernel mean embedding is a useful tool to represent and compare probability measures. 
Despite its usefulness, kernel mean embedding considers infinite-dimensional features, which are challenging to handle in the context of \textit{differentially private} data generation.
A recent work \cite{dpmerf} proposes to approximate the kernel mean embedding of data distribution using \textit{finite-dimensional random features}, which yields analytically tractable sensitivity. 
However, the number of required random features is excessively high, often ten thousand to a hundred thousand, which worsens the privacy-accuracy trade-off. To improve the trade-off, we propose to replace random features with \textit{Hermite polynomial} features. 
Unlike the random features, the Hermite polynomial features are \textit{ordered}, 
where the features at the low orders contain more information on the distribution than those at the high orders. Hence, a relatively low order of Hermite polynomial features can more accurately approximate the mean embedding of the data distribution compared to a significantly higher number of random features. As demonstrated on several 
tabular and image datasets, Hermite polynomial features seem better suited for private data generation than random Fourier features.
\end{abstract}

\section{Introduction}

One of the popular distance metrics for generative modelling is \textit{Maximum Mean Discrepancy} (MMD) \cite{Gretton2012}. MMD computes the average distance between the realizations of two distributions mapped to a reproducing kernel Hilbert space (RKHS). Its popularity is due to several facts: (a) MMD can compare two probability measures in terms of all possible moments (i.e., infinite-dimensional features), resulting in no information loss due to a particular selection of moments; and (b) estimating MMD does not require the knowledge of the probability density functions. Rather, MMD estimators are in closed form, which can be computed by pair-wise evaluations of a kernel function using the points drawn from two distributions.
 
However, 
%
using the MMD estimators 
for training a generator is not well suited when \textit{differential privacy} (DP) of the generated samples is taken into consideration. 
In fact, the generated points are updated in every training step and the pair-wise evaluations of the kernel function on generated and true data points require accessing data multiple times. One of the key properties of DP is composability that implies each access of data causes privacy loss.  Hence, privatizing the MMD estimator in every training step -- which is necessary to ensure the resulting generated samples are differentially private --  incurs a large privacy loss. 

A recent work \cite{dpmerf}, called \textit{DP-MERF}, uses a particular {form} of MMD via a \textit{random Fourier feature} representation \cite{rahimi2008random} of kernel mean embeddings for DP data generation. 
Under this representation, one can approximate the MMD in terms of two finite-dimensional mean embeddings (as in \eqref{MMD_rf}), where the approximate mean embedding of the true data distribution (data-dependent) is detached from that of the synthetic data distribution (data-independent). 
Thus, the data-dependent term needs privatization only once and can be re-used repeatedly during training of a generator.  
However, DP-MERF requires an excessively high number of random features to approximate the mean embedding of data distributions.  




We propose to replace\footnote{There are efforts on improving the efficiency of randomized Fourier feature maps, e.g., by using quasi-random points in \cite{JMLR:v17:14-538}.} the random feature representation of the kernel mean embedding with the \textit{Hermite polynomial} representation.
%
%
We observe that Hermite polynomial features are ordered where the features at the low orders contain more information on the distribution than those at the high orders.
Hence, the required order of Hermite polynomial features is significantly lower than the required number of random features, for the similar quality of the kernel approximation (see \figref{err_order}). 
This is useful in reducing the \textit{effective sensitivity} of the data mean embedding. Although the sensitivity is $\frac{1}{m}$ in both cases with the number of data samples $m$ (see \secref{Methods}), adding noise to a vector of longer length (when using random features) has a worse signal-to-noise ratio, as opposed to adding noise to a vector of shorter length (when using Hermite polynomial features), if we require the norms of these vectors to be the same (for a limited sensitivity). 
Furthermore, the Hermite polynomial features maintain a better signal-to-noise ratio as it contains more information on the data distribution, even when Hermite polynomial features are the same length as the random Fourier features

To this end, we develop a private data generation paradigm, called \textit{differentially private Hermite polynomials} (DP-HP), which utilizes a novel kernel which we approximate with Hermite polynomial features in the aim of effectively tackling the privacy-accuracy trade-off. In terms of three different metrics we use to quantify the quality of generated samples, our method outperforms the state-of-the-art private data generation methods at the same privacy level. What comes next describes relevant background information before we introduce our method. 
    
    


\section{Background}
In the following, we describe the background on kernel mean embeddings and differential privacy.

\subsection{Maximum Mean Discrepancy} \label{sec:mmd}

Given a positive definite kernel
$k\colon\mathcal{X}\times\mathcal{X}$, the MMD between two distributions $P,Q$ is defined as \cite{Gretton2012}: 
$\mathrm{MMD}^2(P,Q)=
 \mathbb{E}_{x,x'\sim P}k(x,x')+\mathbb{E}_{y,y'\sim Q}k(y,y')
 -2\mathbb{E}_{x\sim P}\mathbb{E}_{y\sim Q}k(x,y).$
%
According to the Moore--Aronszajn theorem \cite{aronszajn1950theory},
there exists a unique reproducing kernel Hilbert space of functions on $\mathcal{X}$ for which k is a reproducing kernel, i.e., $k(x,\cdot) \in \mathcal{H}$ and $f(x) = \langle f, k(x,\cdot) \rangle_\mathcal{H}$ for all $x\in \mathcal{X}$ and $f\in \mathcal{H}$, where $\left\langle \cdot,\cdot\right\rangle _{\mathcal{H}}=\left\langle \cdot,\cdot\right\rangle $
denotes the inner product on $\mathcal{H}$.
%
%
%
Hence, we can find a \textit{feature map},  $\phi\colon\mathcal{X}\to\mathcal{H}$
such that $k(x,y)=\left\langle \phi(x),\phi(y)\right\rangle _{\mathcal{H}}$, which allows us to rewrite MMD as \cite{Gretton2012}:
\begin{align}
\mathrm{MMD}^2(P,Q) & =\big\|\mathbb{E}_{x\sim P}[\phi(x)]-\mathbb{E}_{y\sim Q}[\phi(y)]\big\|^2_{\mathcal{H}}, \label{eq:norm2_mmd}
\end{align}
where $\mathbb{E}_{x\sim P}[\phi(x)]\in\mathcal{H}$ is known as the
(kernel) mean embedding of $P$, and exists if $\mathbb{E}_{x\sim P}\sqrt{k(x,x)}<\infty$
\cite{Smola2007}. 
If $k$ is
\textit{characteristic} \cite{Sriperumbudur2011}, then $P\mapsto\mathbb{E}_{x\sim P}[\phi(x)]$
is injective, meaning $\mathrm{MMD}(P,Q)=0$, if and only if $P=Q$. 
Hence, the MMD associated with a characteristic kernel (e.g., Gaussian kernel) can be interpreted as a distance
between the mean embeddings of two distributions.

Given the samples drawn from two distributions: $X_{m}=\{x_{i}\}_{i=1}^{m} \sim P$ and $X'_{n}=\{x'_{i}\}_{i=1}^{n} \sim Q$, we can estimate\footnote{This particular MMD estimator is biased.} the MMD by sample averages \cite{Gretton2012}:
\begin{align}\label{eq:MMD_full}
&\widehat{\mathrm{MMD}}^2(X_{m},X'_{n}) = \tfrac{1}{m^2}\sum_{i,j=1}^{m}k(x_{i},x_{j})\nonumber \\
& \qquad +\tfrac{1}{n^2}\sum_{i,j=1}^{n}k(x'_{i},x'_{j}) 
 -\tfrac{2}{mn}\sum_{i=1}^{m}\sum_{j=1}^{n}k(x_{i},x'_{j}).
\end{align}
However, at $O(mn)$ the computational cost of $
\widehat{\mathrm{MMD}}(X_{m},X'_{n}) $ is prohibitive for large-scale datasets.

\begin{figure}
\centering
\includegraphics[width=0.8\columnwidth]{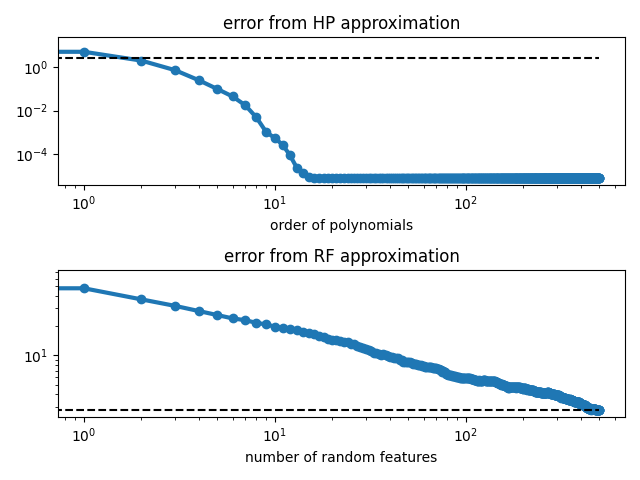}
\caption{\textbf{HP VS.\ RF features.} Dataset $X$ contains $N=100$ samples drawn from  $\Nrm(0,1)$ and $X'$ contains $N=100$ samples drawn from $\Nrm(1,1)$.
The error is defined by:
$ \frac{1}{N^2} \sum_{i=1}^N \sum_{j=1}^N | k(x_i,x_j') - \hat\vphi(x_i)\trp \hat\vphi(x_j')| $
where $\hat\vphi$ is either RF or HP features. 
 \textbf{Top}: The error decays fast when using HP features (\eqref{HP}).   \textbf{Bottom}: The plot shows the average error over $100 $ independent draws of RF features  (\eqref{RF}). The error decays slowly when using RF features. The best error (black dotted line) using $500$ RF features coincides with the error using HP features with order $2$ only.
    } \label{fig:err_order}
\end{figure}

\subsection{Kernel approximation}\label{subsec:RFME}

By approximating the kernel function $k(x,x')$ 
with an inner product of finite dimensional feature vectors, i.e.,  $k(x,x')\approx \hat{\vphi}(x)^{\top}\hat{\vphi}(x')$ where $\hat{\vphi}(x)\in\mathbb{R}^{A}$ and $A$ is the number of
features,
the MMD estimator given in \eqref{MMD_full} can be computed in $O(m+n)$, i.e., linear in the sample size:
\begin{align}\label{eq:MMD_rf}
\widehat{\mathrm{MMD}}^{2}(P,Q)=\bigg\|\tfrac{1}{m}\sum_{i=1}^{m}\hat{\vphi}(x
_i)-\tfrac{1}{n}\sum_{i=1}^{n}\hat{\vphi}(x'_i)\bigg\|_{2}^{2}.
\end{align}
This approximation is also beneficial for private data generation: assuming $P$ is a data distribution and $Q$ is a synthetic data distribution, we can summarize the data distribution in terms of its kernel mean embedding (i.e., the first term on the right-hand side of \eqref{MMD_rf}), which can be privatized only once and used repeatedly during training of the generator which produces samples from $Q$. 

\subsection{Random Fourier features.}
As an example of $\hat{\vphi}(\cdot)$,
the random Fourier features \cite{rahimi2008random} are derived from the following. 
Bochner's theorem \cite{Rudin2013} states that for any translation invariant kernel, the kernel can be written as $k(x,x')=\tilde{k}(x-x')
 =\mathbb{E}_{\omega\sim\Lambda}\cos(\omega^{\top}(x-x')).
$
By drawing
random frequencies $\{\omega_{i}\}_{i=1}^{A}\sim\Lambda$, where $\Lambda$ depends on the kernel,
(e.g., a Gaussian kernel $k$ corresponds to
normal distribution $\Lambda$), 
$\tilde{k}(x-x')$ can be
approximated with a Monte Carlo average. The resulting vector of random Fourier features (of length $A$) is given by 
\begin{align}\label{eq:RF}
    \hat{\vphi}_{RF, \boldsymbol{\omega}}(x)=(\hat{\phi}_{1, \boldsymbol{\omega}}(x),\ldots,\hat{\phi}_{A, \boldsymbol{\omega}}(x))^{\top}
\end{align} where 
$
    \hat\phi_{j, \boldsymbol{\omega}}(x)  = \sqrt{2/A}\;\cos(\omega_j\trp x), 
     \; \hat{\phi}_{j+A/2, \boldsymbol{\omega}}(x) =\sqrt{2/A}\sin(\omega_{j}^{\top}x), \nonumber 
$ for $j=1, \cdots, A/2$.

DP-MERF \cite{dpmerf} uses this very representation of  the feature map given in \eqref{RF}, and minimizes \eqref{MMD_rf} with a privatized data mean embedding to train a generator.  

\subsection{Hermite polynomial features.}
For another example of $\hat{\vphi}(\cdot)$, one could also start with the \textit{Mercer's theorem} (See \suppsecref{generalized}), which allows us to express a positive definite kernel $k$ in terms of the eigen-values $\lambda_i$ and eigen-functions $f_i$: $k(x, x') = \sum_{i=1}^\infty \lambda_i f_i(x) f_i^*(x')$, where $\lambda_i > 0$ and complex conjugate is denoted by $*$. The resulting \textit{finite-dimensional} feature vector is simply $\hat{\vphi}(x) = \hat\vphi_{HP}(x)
= [\sqrt{\lambda_0} f_0(x), \sqrt{\lambda_1} f_1(x), \cdots, \sqrt{\lambda_C} f_C(x)]$, where the cut-off is made at the $C$-th eigen-value and eigen-function.
For the commonly-used Gaussian kernel, $k(x, x') = \exp(-\frac{1}{2 l^2} (x-x')^2)$, where $l$ is the length scale parameter,  an analytic form of eigen-values and eigen-functions are available, where the eigen-functions are represented with Hermite polynomials (See \secref{Methods} for definition). This is the approximation we will use in our method. 


\subsection{Differential privacy}\label{subsec:DP}

Given privacy parameters $\epsilon \geq 0$ and $\delta \geq 0$, a mechanism $\mathcal{M}$ is  ($\epsilon$, $\delta$)-DP if the following equation holds:
$
\Pr[\mathcal{M}(\Dat) \in S] \leq e^\epsilon \cdot \Pr[\mathcal{M}(\Dat') \in S] + \delta,
$
 for all possible sets of the mechanism's outputs $S$ and all neighbouring datasets $\Dat$, $\Dat'$ differing by a single entry.
%
%
In this paper, we use the \textit{Gaussian mechanism} to ensure the output of our algorithm is DP. Consider a function $h: \Dat \mapsto \mathbb{R}^p$, where we add noise for privacy and the level of noise is calibrated to the {\it{global sensitivity}}
\cite{dwork2006our}, $\Delta_h$, defined by the maximum difference in terms of $L_2$-norm $||h(\Dat)-h(\Dat') ||_2$, for neighbouring $\Dat$ and $\Dat'$ (i.e. $\Dat$ and $\Dat'$ have one sample difference by replacement). where the output is denoted by $\widetilde{h}(\Dat) = h(\Dat) + n$, where $n\sim \Nrm(0, \sigma^2 \Delta_h^2\mathbf{I}_p)$.
The perturbed function $\widetilde{h}(\Dat) $ is $(\epsilon, \delta)$-DP, where $\sigma$ is a function of $\epsilon$ and $\delta$ and can be numerically computed using, e.g.,  
the auto-dp package by \cite{wang2019subsampled}.







\section{Our method: DP-HP}\label{sec:Methods}

\subsection{Approximating the Gaussian kernel using Hermite polynomials (HP)}\label{sec:HP}
Using the \textit{Mehler formula}\footnote{This formula can be also derived from the Mercer's theorem as shown in \cite{zhu1997gaussian, GPML}.} \cite{mehler1866ueber}, 
for $|\rho|<1$, we can write down the Gaussian kernel\footnote{The length scale $l$  in terms of $\rho$ is $\frac{1}{2 l^2} = \frac{\rho}{1-\rho^2}$.} as a weighted sum of Hermite polynomials
\begin{equation}\label{eq:mehler}
    \exp\left( - \frac{\rho}{1-\rho^2} (x-y)^2 \right) 
    = \sum_{c=0}^\infty \lambda_c f_c(x) f_c(y) 
\end{equation} where the $c$-th eigen-value is $\lambda_c=(1-\rho) \rho^c$ and the $c$-th eigen-function is defined by $f_c$, where 
$ 
    f_c(x) = \frac{1}{\sqrt{N_c}} H_c(x) \exp\left( - \frac{\rho}{1+\rho} x^2 \right), 
$
and 
$    N_c = 2^c c! \sqrt{\frac{1-\rho}{1+\rho}}$. 
Here, $H_c(x) = (-1)^c \exp(x^2)\frac{d^c}{dx^c}\exp(-x^2)$ is the $c$-th order physicist's Hermite polynomial. 

As a result of the Mehler formula, we can define a $C$-th order Hermite polynomial features as a feature map (a vector of length $C+1$):  
\begin{align}\label{eq:HP}
\hat\vphi^{(C)}_{HP}(x) 
= \left[ \sqrt{\lambda_0} f_0(x),   \cdots, \sqrt{\lambda_C} f_C(x) \right],
\end{align} and approximate the Gaussian kernel via
$
    \exp\left( - \frac{\rho}{1-\rho^2} (x-y)^2 \right) \approx \hat\vphi^{(C)}_{HP}(x)\trp \hat\vphi^{(C)}_{HP}(y).
$

This feature map provides us with a uniform approximation to the MMD in \eqref{norm2_mmd}, for every pair of distributions $P$ and $Q$ (see Theorem \ref{thm:marcer} and Lemma \ref{lem:unif_mmd} in \suppsecref{generalized}).

We compare
the accuracy of this approximation 
with random features in \figref{err_order}, where we fix the length scale to the median heuristic value\footnote{Median heuristic is a commonly-used heuristic to choose a length scale, which picks a value in the middle range (i.e., median) of $\|x_i - x_j \|$ for $1 \leq i,j \leq n$ for the dataset of $n$ samples.} in both cases. 
Note that the bottom plot shows the average error across $100$ independent draws of random Fourier features. 
We observe that the error decay is significantly faster when using HPs than using RFs.
For completeness, we derive the kernel approximation error under HP features and random features for 1-dimensional data in \suppsecref{prf_conv}. 
Additionally, 
we visualize the effect of length scale on the error further in \suppsecref{compare_HP}. 



\paragraph{Computing the Hermite polynomial features.}
Hermite polynomials follow the recursive definition: $H_{c+1}(x) = 2x H_c(x)  - 2c  H_{c-1}(x)$. At high orders, the polynomials take on large values, leading to numerical instability. So we compute the re-scaled term $\phi_c = \sqrt{\lambda_c} f_c$ iteratively using a similar recursive expression given in \suppsecref{recursion}.

\subsection{Handling multi-dimensional inputs}\label{sec:sum_kernel}

%

\subsubsection{Tensor (or outer) product kernel}
The Mehler formula holds for 1-dimensional input space.
For $D$-dimensional inputs  $\vx, \vx' \in \mathbb{R}^D$, where $\vx = [x_1, \cdots, x_D]$ and $\vx' = [x'_1, \cdots, x'_D]$, the \textit{generalized Hermite Polynomials} 
(Proposition \ref{prop:mahal} and Remark \ref{rem: unif_gen} in \suppsecref{generalized}) allows us to represent the multivariate Gaussian kernel $k(\vx, \vx')$ by
a tensor (or outer) products of the Gaussian kernel defined on each input dimension, where the coordinate-wise Gaussian kernel is approximated with Hermite polynomials:
\begin{align}
    k(\vx, \vx') &= k_{X_1} \otimes k_{X_2} \cdots \otimes k_{X_D} =  \prod_{d=1}^D k_{X_d}(x_d, x'_d), \nonumber \\
    &\approx \prod_{d=1}^D  \hat{\vphi}_{HP}^{(C)}(x_{d})\trp  \hat{\vphi}_{HP}^{(C)}(x_{d}),  
\end{align} where $\hat{\vphi}_{HP}^{(C)}(.)$\footnote{One can let each coordinate's Hermite Polynomials $\phi_{HP, d}^{(C)}(x_d)$ take different values of $\rho$, which determine a different level of fall-offs of the  eigen-values and a different range of values of the eigen-functions. Imposing a different cut-off $C$ for each coordinate is also possible.} is defined in \eqref{HP}. 
The corresponding feature map, from $k(\vx, \vx') \approx \vh_p(\vx) \trp \vh_p(\vx')$, is written as 
\begin{align}\label{eq:full_feature_map_prod}
    &\vh_p(\vx) \nonumber \\
    &=   \mbox{vec} \left[  
    \hat{\vphi}_{HP}^{(C)}(x_{1})  \otimes
    \hat{\vphi}_{HP}^{(C)}(x_{2})  \otimes
    \cdots
    \hat{\vphi}_{HP}^{(C)}(x_{D})\right]  
\end{align} where $\otimes$ denotes the tensor (outer) product  and $\mbox{vec}$ is an operation that vectorizes a tensor. 
The size of the feature map is $(C+1)^D$, where $D$ is the input dimension of the data and $C$ is the chosen order of the Hermite polynomials.
%
This is prohibitive for the datasets we often deal with, e.g., for MNIST ($D=784$) with a relatively small order (say $C=10$), the size of feature map is $11^{784}$, impossible to fit in a typical size of memory. 

In order to handle high-dimensional data in a computationally feasible manner, we propose the following approximation. First we subsample input dimensions where the size of the selected input dimensions is denoted by $D_{prod}$. We then compute the feature map only on those selected input dimensions denoted by  $\vx^{D_{prod}}$. We repeat these two steps during training. 
The size of the feature map becomes $(C+1)^{D_{prod}}$, significantly lower than $(C+1)^D$ if $D_{prod} \ll D$. 
What we lose in return is the injectivity of the Gaussian kernel on the full input distribution, as we compare two distributions in terms of selected input dimensions.
We need a quantity that is more computationally tractable and also helps  distinguishing two distributions, which we describe next.


\subsubsection{Sum kernel}
Here, we define another kernel on the joint distribution over $(x_1, \cdots, x_D)$.
%
The following kernel is formed by defining a 1-dimensional Gaussian kernel on each of the input dimensions:
\begin{align}\label{eq:sum_kernel}
    \tilde{k}(\vx, \vx') 
   &=   \tfrac{1}{D} \left[ k_{X_1}(x_1, x_1') +  \cdots + k_{X_D}(x_D, x'_D)\right], \nonumber \\
   &= \tfrac{1}{D} \sum_{d=1}^D k_{X_d}(x_d, x'_d), \nonumber \\
   &\approx \tfrac{1}{D} \sum_{d=1}^D \hat{\vphi}_{HP}^{(C)}(x_{d})\trp  \hat{\vphi}_{HP}^{(C)}(x_{d}),
\end{align} where $\hat\vphi_{HP, d}^{(C)}(.)$ is given in \eqref{HP}. 
%
%
The corresponding feature map, from $\tilde{k}(\vx, \vx') \approx \vh_s(\vx) \trp \vh_s(\vx')$, is represented by 
\begin{align}\label{eq:feature_map}
    \vh_s(\vx) &=     \begin{bmatrix} 
    \hat\vphi^{(C)}_{HP,1}(x_{1})/\sqrt{D}   \\
    \hat\vphi^{(C)}_{HP,2}(x_{2})/\sqrt{D}   \\
    \vdots \\
     \hat\vphi^{(C)}_{HP,D}(x_{D})/\sqrt{D} 
    \end{bmatrix}   \in \mathbb{R}^{( (C+1) \cdot D) \times 1},
\end{align} where the features map is the size of $(C+1)D$.
For the MNIST digit data ($D=784$), with a relatively small order, say $C=10$, the size of the feature map is $11 \times {784}=8624$ dimensional, which is manageable compared to the size ($11^{784}$) of the feature map under the generalized Hermite polynomials. 
%

Note that the sum kernel does not approximate the Gaussian kernel defined on the joint distribution over all the input  dimensions. Rather, the assigned Gaussian kernel \textit{on each dimension is characteristic}. 
The Lemma \ref{lem:marginal_mmd} in \suppsecref{sum_kernel} shows that by minimizing the approximate MMD between the real and synthetic data distributions based on feature maps given in \eqref{feature_map}, we assure that the marginal probability distributions of the synthetic data converges to those of the real data.

\subsubsection{Combined Kernel}


Finally we arrive at a new kernel, which comes from a sum of the two fore-mentioned kernels:
\begin{align}\label{eq:final_kernel}
k_c(\vx, \vx') = k(\vx, \vx') + \tilde{k}(\vx, \vx'),
\end{align} where $k(\vx, \vx')  \approx \vh_p(\vx^{D_{prod}})\trp \vh_p(\vx'^{D_{prod}})$ and $ \tilde{k}(\vx, \vx') \approx \vh_s(\vx)\trp\vh_s(\vx')$, and consequently the corresponding feature map  is given by 
\begin{align}\label{eq:final_feature_map}
    \vh_c(\vx) &=    
    \begin{bmatrix} 
   \vh_p(\vx^{D_{prod}})\\
    \vh_s(\vx)
    \end{bmatrix}  
\end{align} where the size of the feature map is $ \mathbb{R}^{\left( (C+1)^{D_{prod}} + (C+1)\cdot D) \right ) \times 1}$. 

\textbf{Why this kernel?} When $D_{prod}$ goes to $D$, the product kernel itself in \eqref{final_kernel} becomes characteristic, which allows us to reliably compare two distributions. However, for computational tractability, we are restricted to choose a relatively small $D_{prod}$ to subsample the input dimensions, which forces us to lose information on the distribution over the un-selected input dimensions. The use of sum kernel is to provide extra  information on the un-selected input dimensions at a particular training step. Under our kernel in \eqref{final_kernel}, every input dimension's marginal distributions are compared between two distributions in all the training steps due to the sum kernel, while some of the input dimensions are chosen to be considered for more detailed comparison (e.g., high-order correlations between selected input dimensions) due to the outer product kernel.



\subsection{Approximate MMD for classification}
\label{sec:approx_mmd_for_classification}
For classification tasks, we define a mean embedding for the joint distribution over the input and output pairs $(\vx, \vy)$, with the particular feature map given by $\vg$
\begin{align}\label{eq:me_classification}
     \widehat{\vmu}_{P_{\vx, \vy}}(\Dat) &=  \tfrac{1}{m}\sum_{i=1}^{m}\mathbf{g}(\vx_i, \vy_i).
\end{align} Here, we define the feature map as an outer product between the input features represented by \eqref{final_feature_map} and the output labels represented by one-hot-encoding $\vf(\vy_i)$:
\begin{align}\label{eq:feature_map_classification}
\mathbf{g}(\vx_i, \vy_i) = \vh_c(\vx_i) \vf(\vy_i)^T. 
\end{align} Given \eqref{feature_map_classification}, we further decompose \eqref{me_classification} into two, where the first term corresponds to the outer product kernel denoted by $\widehat{\vmu}^p_P$  and the second term corresponds to the sum kernel denoted by $\widehat{\vmu}^s_P$:
\begin{align}\label{eq:final_feature_map_two_terms}
   \widehat{\vmu}_{P_{\vx, \vy}} =&    
    \begin{bmatrix} 
   \widehat{\vmu}^p_P \\
    \widehat{\vmu}^s_P
    \end{bmatrix} 
    = \begin{bmatrix} \tfrac{1}{m}\sum_{i=1}^{m}\vh_p(\vx_i^{D_{prod}})\vf(\vy_i)^T \\
    \tfrac{1}{m}\sum_{i=1}^{m} \vh_s(\vx_i) \vf(\vy_i)^T 
    \end{bmatrix}. 
\end{align} 

Similarly, we define an approximate mean embedding of the synthetic data distribution by $\widehat{\vmu}_{Q_{\vx', \vy'}}(\Dat'_\vtheta) =  \tfrac{1}{n}\sum_{i=1}^{n}\mathbf{g}(\vx'_i(\vtheta), \vy'_i(\vtheta))$,  where $\vtheta$ denotes the parameters of a synthetic data generator.
Then, the approximate MMD is given by: 
$\widehat{\mathrm{MMD}}_{HP}^{2}(P,Q)=||\widehat{\vmu}_{P_{\vx, \vy}}(\Dat)-\widehat{\vmu}_{Q_{\vx', \vy'}}(\Dat'_\vtheta)||_{2}^{2}
= ||\widehat{\vmu}^p_P - \widehat{\vmu}^p_{Q_\vtheta}||_{2}^{2} + ||\widehat{\vmu}^s_P - \widehat{\vmu}^s_{Q_\vtheta} ||_{2}^{2}. $
In practice, we minimize the augmented approximate MMD:
\begin{align}\label{eq:aug_mmd_classification}
\min_{\vtheta } \; \gamma ||\widehat{\vmu}^p_P - \widehat{\vmu}^p_{Q_\vtheta}||_{2}^{2} + ||\widehat{\vmu}^s_P - \widehat{\vmu}^s_{Q_\vtheta} ||_{2}^{2}.
\end{align} where $\gamma$ is a positive constant (a hyperparameter) that helps us to deal with the scale difference in the two terms (depending on the selected HP orders and subsampled input dimensions) and also allows us to give a different importance on one of the two terms. 
We provide the details on how $\gamma$ plays a role and whether the algorithm is sensitive to $\gamma$ in \secref{experiments}.
%
Minimizing \eqref{aug_mmd_classification} yields a synthetic data distribution over the input and output, which minimizes the discrepancy in terms of the particular feature map \eqref{final_feature_map_two_terms}  between synthetic and real data distributions. 



%


\subsection{Differentially private  data samples}

For obtaining privacy-preserving synthetic data, 
all we need to do is privatizing  $\widehat{\vmu}^p_P$  and  $\widehat{\vmu}^s_P$  given in \eqref{final_feature_map_two_terms}, then training a generator. 
We use the Gaussian mechanism to privatize both terms. 
See \suppsecref{sensitivity} for  sensitivity analysis.
Unlike $\widehat{\vmu}^s_P$ that can be privatized only and for all, we need to privatize $\widehat{\vmu}^p_P$ every time we redraw the subsampled input dimensions. We split a target $\epsilon$ into two such that $\epsilon = \epsilon_1 + \epsilon_2$ (also the same for $\delta$), where $\epsilon_1$ is used for privatizing  $\widehat{\vmu}^s_P$ and $\epsilon_2$ is used for privatizing 
$\widehat{\vmu}^p_P$. 
We further compose the privacy loss incurred in privatizing $\widehat{\vmu}^p_P$ during training by the analytic moments accountant \cite{wang2019subsampled}, which returns the privacy parameter $\sigma$ as a function of $(\epsilon_2, \delta_2)$. 
In the experiments, we subsample the input dimensions for the outer product kernel in every epoch as opposed to in every training step for an economical use of $\epsilon_2$.

 \begin{figure*}[htb]
    \centering
    \includegraphics[scale=1.]{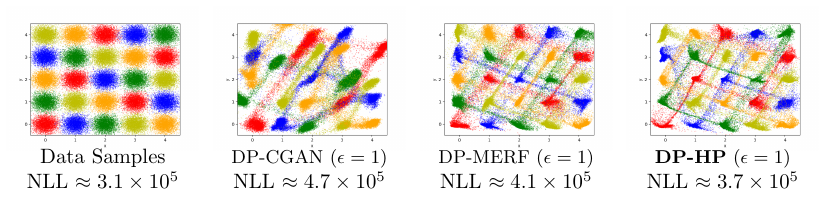}
    \caption{Simulated example from a Gaussian mixture.  \textbf{Left}:  Data samples drawn from a Gaussian Mixture distribution with 5 classes (each color represents a class). NLL denotes the negative log likelihood of the samples given the true data distribution. 
    \textbf{Middle-Left}:
    Synthetic data generated by DP-CGANs at $\epsilon=1$, where some modes are dropped, which is reflected in poor NLL. 
    \textbf{Middle-Right}: Synthetic data samples generated by DP-MERF at $\epsilon=1$. 
    \textbf{Right}: Synthetic data samples generated by DP-HP at $\epsilon=1$. 
    Our method captures all modes accurately at $\epsilon=1$, and achieves better NLL thanks to a smaller size of feature map than that of DP-MERF (see text).
    }
    \label{fig:syn2d_samples}
\end{figure*}

\section{Related Work}

Approaches to differentially private data release can be broadly sorted into three categories. 
One line of prior work with background in learning theory aims to provide theoretical guarantees on the utility of released data \cite{snoke2018pmse, Mohammed_DPDR_for_data_mining, Xiao_DPDR_trough_partitioning, Hardt_simple_practical_DPDR, Zhu_DPDP_survey}.
This usually requires strong constraints on the type of data and the intended use of the released data.

A second line of work focuses on the sub-problem of discrete data with limited domain size, which is relevant to tabular datasets \cite{privbayes, priview, chen2015dpdatapublication, zhang2021privsyn}. Such approaches typically approximate the structure in the data by identifying small sub-sets of features with high correlation and releasing these lower order marginals in a private way.
Some of these methods have also been successful in the recent NIST 2018 Differential Privacy Synthetic Data Challenge \cite{nist2018},
while these methods often require discretization of the data and do not scale to higher dimensionality in arbitrary domains. 

The third line of work aims for broad applicability without constraints on the type of data or the kind of downstream tasks to be used. Recent approaches 
attempt to leverage the modeling power of deep generative models in the private setting. While work on VAEs exists \cite{acs2018differentially}, GANs are  the most popular model \cite{DPGAN, DP_CGAN, DBLP:conf/sec/FrigerioOGD19, PATE_GAN, gs-wgan}, where most of these utilize a version of DP-SGD \cite{DP_SGD} to accomplish this training, 
 while PATE-GAN is based on the private aggregation of teacher ensembles (PATE) \cite{papernot:private-training}. 

The closest prior work to the proposed method is DP-MERF \cite{dpmerf}, where kernel mean embeddings are approximated with random Fourier features \cite{rahimi2008random} instead of Hermite polynomials. Random feature approximations of MMD have also been used with DP \cite{BalTolSch18, sarpatwar2019differentially}.
A recent work utilizes the Sinkhorn divergence for private data generation \cite{sinkhorn_21}, which more or less matches the results of DP-MERF when the regularizer is large and the cost function is the L2 distance. 
%
To our knowledge, ours is the first work using Hermite polynomials to approximate MMD in the context of differentially private data generation.


\section{Experiments}\label{sec:experiments}

Here, we show the performance of our method tested on several real world datasets. 
Evaluating the quality of generated data itself is challenging. Popular metrics such as inception score and Fr\'echet inception distance are appropriate to use for evaluating color images.
For the generated samples for tabular data and black and white images, 
we use the following three metrics:
    (a) Negative log-likelihood of generated samples given a ground truth model in \secref{2d_MoG};  
    (b) $\alpha$-way marginals of generated samples in \secref{discrete_data} to judge whether the generated samples contain a similar correlation structure to the real data;
    (c) Test accuracy on the real data given classifiers trained with generated samples in \secref{generalization_error} to judge the generalization performance from synthetic to real data. 

As comparison methods, we tested PrivBayes \cite{privbayes}, DP-CGAN \cite{DP_CGAN}, DP-GAN \cite{DPGAN} and DP-MERF \cite{dpmerf}. For image datasets we also trained GS-WGAN \cite{gs-wgan}. 
Our experiments were implemented in PyTorch \cite{pytorch} and run using Nvidia Kepler20 and Kepler80 GPUs. Our code is available at
\url{https://github.com/ParkLabML/DP-HP}.

\begin{table*}[htb]
\caption{$\alpha$-way marginals evaluated on generated samples with discretized Adult and Census datasets.}
\centering
\scalebox{0.8}{
\begin{tabular}{lc|cc|cc|cc || lc|cc|cc|cc}
\toprule
\multicolumn{2}{ c| }{{\textit{Adult}}} & \multicolumn{2}{ c| }{PrivBayes} & \multicolumn{2}{ c| }{DP-MERF} & 
\multicolumn{2}{ c|| }{\textbf{DP-HP}} & 
\multicolumn{2}{ c| }{{\textit{Census}}} & \multicolumn{2}{ c| }{PrivBayes} & \multicolumn{2}{c|}{DP-MERF}&
\multicolumn{2}{c}{DP-HP} \\ 
& &  $\epsilon {=} 0.3$ & $\epsilon{=}0.1$ &  $\epsilon{=}0.3$ & $\epsilon{=}0.1$ & $\epsilon{=}0.3$ & $\epsilon{=}0.1$
& & & $\epsilon{=}0.3$ & $\epsilon{=}0.1$ &  $\epsilon{=}0.3$ & $\epsilon{=}0.1$ & $\epsilon{=}0.3$ & $\epsilon{=}0.1$
\\
\midrule
{} & $\alpha{=}3$ & 0.446 & 0.577 & 0.405 & 0.480 & \textbf{0.332} & \textbf{0.377}
& {} & $\alpha{=}2$ &  0.180 & 0.291 & 0.190 & 0.222 &  \textbf{0.141} & \textbf{0.155}
\\
& $\alpha{=}4$ & 0.547 & 0.673 & 0.508 & 0.590 & \textbf{0.418} & \textbf{0.467}
& & $\alpha{=}3$ & 0.323 & 0.429 & 0.302 & 0.337 & \textbf{0.211} & \textbf{0.232}
\\
\bottomrule
\label{tab:alpha_way}
\end{tabular}} 
\end{table*}

\subsection{2D Gaussian mixtures}\label{sec:2d_MoG}
We begin our experiments on Gaussian mixtures, as shown in \figref{syn2d_samples} (left). We generate 4000 samples from each Gaussian, reserving 10\% for the test set,
which yields 90000 training samples from the following distribution:
$p(\vx, \vy) = \prod_i^N \sum_{j \in C_{\vy_i}} \frac{1}{C} \mathcal{N}(\vx_i|\vmu_j, \sigma \mI_2)$
where $N=90000$, and $\sigma=0.2$. $C=25$ is the number of clusters and $C_y$ denotes the set of indices for means $\vmu$ assigned to class $y$. Five Gaussians are assigned to each class, which leads to a uniform distribution over $\vy$ and $18000$ samples per class.
We use the negative log likelihood (NLL) 
of the samples under the true distribution as a score\footnote{Note that this is different from the other common measure of computing the negative log-likelihood of the true data given the learned model parameters.} to measure the quality of the generated samples: $\text{NLL}(\vx, \vy) =-\log p(\vx, \vy)$. The lower NLL the better. 
%
%

We compare our method to DP-CGAN and DP-MERF at $(\epsilon, \delta)=(1, 10^{-5})$ in \figref{syn2d_samples}.
Many of the generated samples by DP-CGAN  fall out of the distribution and some modes are dropped (like the green one in the top right corner). DP-MERF  preserves all modes. DP-HP performs better than DP-MERF by placing fewer samples in low density regions as indicated by the low NLL. This is due to the drastic difference in the size of the feature map. DP-MERF used $30,000$ random features (i.e., $30,000$-dimensional feature map). DP-HP used the $25$-th order Hermite polynomials on both sum and product kernel approximation (i.e., $25^2+25 = 650$-dimensional feature map). in this example, as the input is 2-dimensional, it was not necessary to subsample the input dimensions to approximate the outer product kernel.

\subsection{$\alpha$-way marginals with discretized tabular data}\label{sec:discrete_data} 

We compare our method to PrivBayes \cite{privbayes} and DP-MERF. For PrivBayes, we used the  published code from \cite{mckenna2019graphical}, which builds on the original code with \cite{zhang2018ektelo} as a wrapper. 
We test the model on the discretized Adult and Census datasets. Although these datasets are typically used for classification, we use their inputs only for the task of learning the input distribution.
Following \cite{privbayes}, we measure $\alpha$-way marginals of generated samples at varying levels of $\epsilon$-DP with $\delta=10^{-5}$. 
We measure the accuracy of each marginal of the generated dataset by the total variation distance between itself and the real data marginal (i.e., half of the L1 distance between the two marginals, when both of them are treated as probability distributions). We use the average accuracy over all marginals as the final error metric for $\alpha$-way marginals.
In \tabref{alpha_way}, our method outperforms other two at the stringent privacy regime. See \suppsecref{hyperparams_discrete} for hyperparameter values we used, and  \suppsecref{gamma} for the impact of $\gamma$ on the quality of the generated samples. We also show how the selection of $D_{prod}$ affects the accuracy in  \suppsecref{tradeoff_subsamp_dims}.  

\subsection{Generalization from synthetic to real data}\label{sec:generalization_error}

Following \cite{gs-wgan, DP_CGAN, PATE_GAN,gs-wgan, dpmerf, sinkhorn_21}, 
we evaluate the quality of the (private and non-private) generated samples from these models using the common approach of measuring performance on downstream tasks.
We train 12 different commonly used classifier models using generated samples and then evaluate the classifiers on a test set containing \textit{real} data samples. 
Each setup is averaged over 5 random seeds. The test accuracy indicates how well the models generalize from the synthetic to the real data distribution and thus, the utility of using private data samples instead of the real ones. Details on the 12 models can be found in \tabref{image_downstream_hyperparam}. 

\paragraph{Tabular data.} First, we explore the performance of DP-HP algorithm on eight  different imbalanced tabular datasets with both numerical and categorical input features. The numerical features on those tabular datasets can be either discrete (e.g. age in years) or continuous (e.g. height) and the categorical ones may be binary (e.g. drug vs placebo group) or multi-class (e.g. nationality). The datasets are described in detail in \suppsecref{app_results}. 
%
As an evaluation metric, we use ROC (area under the receiver characteristics curve) and PRC (area under the precision recall curve) for datasets with binary labels, and F1 score for dataset with multi-class labels.  
\tabref{summary_tabular} shows the average over the 12 classifiers trained on the generated samples (also averaged over 5 independent seeds), where  
overall DP-HP outperforms the other methods in both the private and non-private settings, followed by DP-MERF.\footnote{For the Cervical dataset, the non-privately generated samples by DP-MERF and DP-HP give better results than the baseline trained with real data. This may be due to the fact that the dataset is relatively small which can lead to overfitting. The generating samples by DP-MERF and DP-HP could bring a regularizing effect, which improves the performance as a result.} 
See \suppsecref{tabular_data_DP_HP_hyperparam} for hyperparameter values we used. We also show the non-private MERF and HP results in \tabref{non_priv_as_well_summary_tabular} in Appendix.



\begin{table*}[t!]
\caption{Performance comparison on Tabular datasets. The average over five independent runs. 
}
\centering
\scalebox{0.7}{
\begin{tabular}{l *{5}{|cc} }
\toprule
& \multicolumn{2}{ c| }{Real}  & \multicolumn{2}{ c| }{DP-CGAN}  &
\multicolumn{2}{ c| }{DP-GAN}  &
\multicolumn{2}{ c| }{{DP-MERF}} &
\multicolumn{2}{ c }{\textbf{DP-HP}}\\ 
&\multicolumn{2}{ c| }{} &   \multicolumn{2}{ c| }{($1,10^{-5}$)-DP} & \multicolumn{2}{ c| }{($1,10^{-5}$)-DP} & 
\multicolumn{2}{ c| }{($1,10^{-5}$)-DP} &
\multicolumn{2}{ c }{($1,10^{-5}$)-DP}\\ 
\midrule
\textbf{adult} & 0.786 & 0.683 &  0.509 & 0.444 & 0.511 & 0.445 & 0.642 & 0.524 & \textbf{0,688 }& \textbf{0,632} \\
\textbf{census} & 0.776 & 0.433   & 0.655 & 0.216 & 0.529 & 0.166 & 0.685 & 0.236 & \textbf{0,699} & \textbf{0,328 }\\
\textbf{cervical} & 0.959 & 0.858  & 0.519 & 0.200 & 0.485 & 0.183 & 0.531 & 0.176 & \textbf{0,616 }& \textbf{0,312 }\\
\textbf{credit} & 0.924 & 0.864 &  0.664 & 0.356 & 0.435 & 0.150 & 0.751 & 0.622 & \textbf{0,786} & \textbf{0,744} \\
\textbf{epileptic} & 0.808 & 0.636 &  0.578 & 0.241 & 0.505 & 0.196 & 0.605 & 0.316 & \textbf{0,609} &\textbf{ 0,554} \\
\textbf{isolet} & 0.895 & 0.741 & 0.511 & 0.198 & 0.540 & 0.205 & 0.557 & 0.228 & \textbf{0,572} & \textbf{0,498} \\ 

& \multicolumn{2}{ c| }{F1}
& \multicolumn{2}{ c| }{F1}
& \multicolumn{2}{ c| }{F1}
& \multicolumn{2}{ c| }{F1}
& \multicolumn{2}{ c }{F1}\\ 
\textbf{covtype} & \multicolumn{2}{ c| }{0.820} &
\multicolumn{2}{ c| }{0.285} &
\multicolumn{2}{ c| }{0.492} &
\multicolumn{2}{ c| }{0.467} &  \multicolumn{2}{ c }{\textbf{0.537}}\\
\textbf{intrusion} & \multicolumn{2}{ c| }{0.971} &  
\multicolumn{2}{ c| }{0.302} & 
\multicolumn{2}{ c| }{0.251} &
\multicolumn{2}{ c| }{\textbf{0.892}} & \multicolumn{2}{ c }{0.890} \\ 
\bottomrule
\end{tabular}}\label{tab:summary_tabular}
\end{table*}

\begin{figure*}
 \begin{subfigure}[b]{0.5\textwidth}
 \centering
\includegraphics[width=0.8\linewidth]{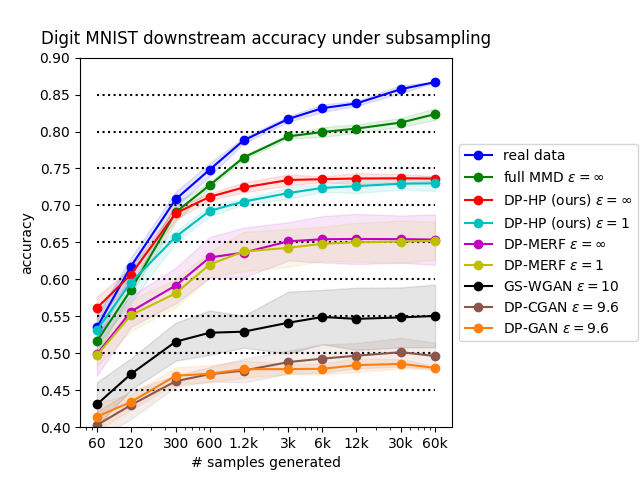}
 \caption{MNIST}
 \label{fig:mnist_subsamp_d}
\end{subfigure}
\begin{subfigure}[b]{0.5\textwidth}
 \centering
\includegraphics[width=0.8\linewidth]{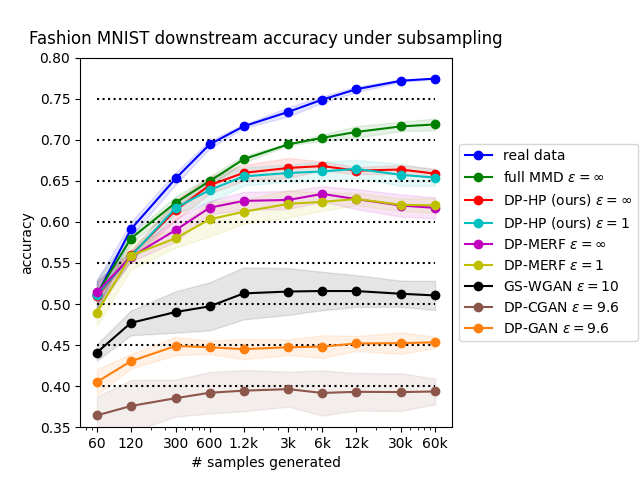}
 \caption{FashionMNIST}
 \label{fig:mnist_subsamp_f}
 \end{subfigure}
\caption{We compare the real data test accuracy as a function of training set size for models trained on synthetic data from DP-HP and comparison models. Confidence intervals show 1 standard deviation.}
\label{fig:mnist_subsamp}
\end{figure*}

\paragraph{Image data.} We follow previous work in testing our method on image datasets MNIST \cite{lecun2010mnist} (license: CC BY-SA 3.0) and FashionMNIST \cite{xiao2017fashion} (license: MIT). Both datasets contain 60000 images from 10 different balanced classes. We test both fully connected and convolutional generator networks and find that the former works better for MNIST, while the latter model achieves better scores on FashionMNIST. For the experimental setup of DP-HP on the image datasets see \tabref{image_hyperparam} in \suppsecref{hyperparams}. A qualitative sample of the generated images for DP-HP and comparison methods is shown in \figref{generated_samples}. 
While qualitatively GS-WGAN produces the cleanest samples, DP-HP outperforms GS-WGAN on downstream tasks. This can be explained by a lack of sample diversity in GS-WGAN shown in \figref{mnist_subsamp}.

In \figref{mnist_subsamp}, we compare the test accuracy on real image data based on private synthetic samples from DP-GAN, DP-CGAN, GS-WGAN, DP-MERF and DP-HP generators. As additional baselines we include performance of real data and of \textit{full MMD}, a non-private generator, which is trained with the MMD estimator in \eqref{MMD_full} in a mini-batch fashion.
DP-HP gives the best accuracy over the other considered methods followed by DP-MERF but with a considerable difference especially on the MNIST dataset. For GAN-based methods, we use the same weak privacy constraints given in the original papers, because they do not produce meaningful samples at $\epsilon=1$. Nonetheless, the accuracy these models achieve remains relatively low.
Results for individual models for both image datasets are given in \suppsecref{image_results}. 

\begin{figure}[t]
    \centering
    \includegraphics[width=0.8\linewidth]{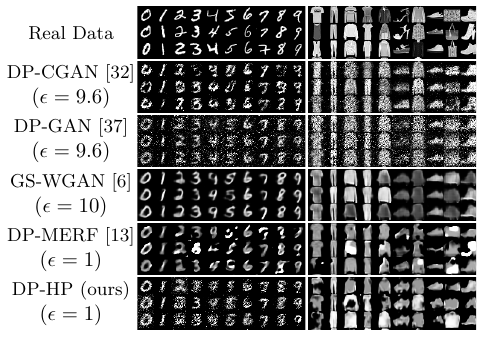}
    \caption{Generated MNIST and FashionMNIST samples from DP-HP and comparison models}
    \label{fig:generated_samples}
\end{figure}
%

Finally, we show the downstream accuracy for smaller generated datasets down to 60 samples (or 0.1\% of original dataset) in \figref{mnist_subsamp}. The points, at which additional generated data does not lead to improved performance, gives us a sense of the redundancy present in the generated data. We observe that all generative models except \textit{full MMD} see little increase in performance as we increase the number of synthetic data samples to train the classifiers. This indicates that the \textit{effective dataset size} these methods produce lies only at about 5\% (3k) to 10\% (6k) of the original data. For DP-GAN and DP-CGAN this effect is even more pronounced, showing little to no gain in accuracy after the first 300 to 600 samples respectively on FashionMNIST.

\section{Summary and Discussion}
We propose a DP data generation framework that improves the privacy-accuracy trade-off using the Hermite polynomials features thanks to the orderedness of the polynomial features. 
We chose the combination of  outer product and sum kernels computational tractability in handling high-dimensional data.
%
The quality of generated data by our method is significantly higher than that by other state-of-the-art methods, in terms of three different evaluation metrics. 
In all experiments, we observed that assigning  $\epsilon$ more to $\epsilon_1$ than $\epsilon_2$ and using the sum kernel's mean embedding as a main objective together with the outer product kernel's mean embedding as a constraint (weighted by $\gamma$) help improving the performance of DP-HP. 

As the size of mean embedding grows exponentially with the input dimension under the outer product kernel, we chose to subsample the input dimensions. However, even with the subsampling, we needed to be careful not to explode the size of the kernel's mean embedding, which limits the subsampling dimension to be less than 5, in practice.
This gives us a question whether there are better ways to approximate the outer product kernel than random sampling across all input dimensions. We leave this for future work. 

\section*{Acknowledgements}
 M.A. Charusaie and F. Harder thank the International Max Planck Research
School for Intelligent Systems (IMPRS-IS) for its support. M. Vinaroz thanks the support by the Gibs Schule Foundation and the Institutional Strategy of the University of T\"ubingen (ZUK63).  K. Adamczewski is grateful for
the support of the Max Planck ETH Center for Learning Systems.
M. Park thanks the CIFAR AI Chair fund (at AMII) for its support. M. Park also thanks Wittawat Jitkrittum and Danica J. Sutherland for their valuable time
discussing the project. 
We thank our anonymous reviewers for their constructive feedback.

\bibliography{ms}
\bibliographystyle{icml2022}

\newpage
\appendix
\onecolumn


\begin{center}
    {\LARGE\textbf{Appendix}}
\end{center}

\section{Effect of length scale on the kernel approximation}\label{supp:compare_HP}

\figref{feature_comparison} shows the effect of the kernel length scale on the kernel approximation for both HPs and RFs.

\begin{figure}[h]
\centering
\includegraphics[width=0.8\linewidth]{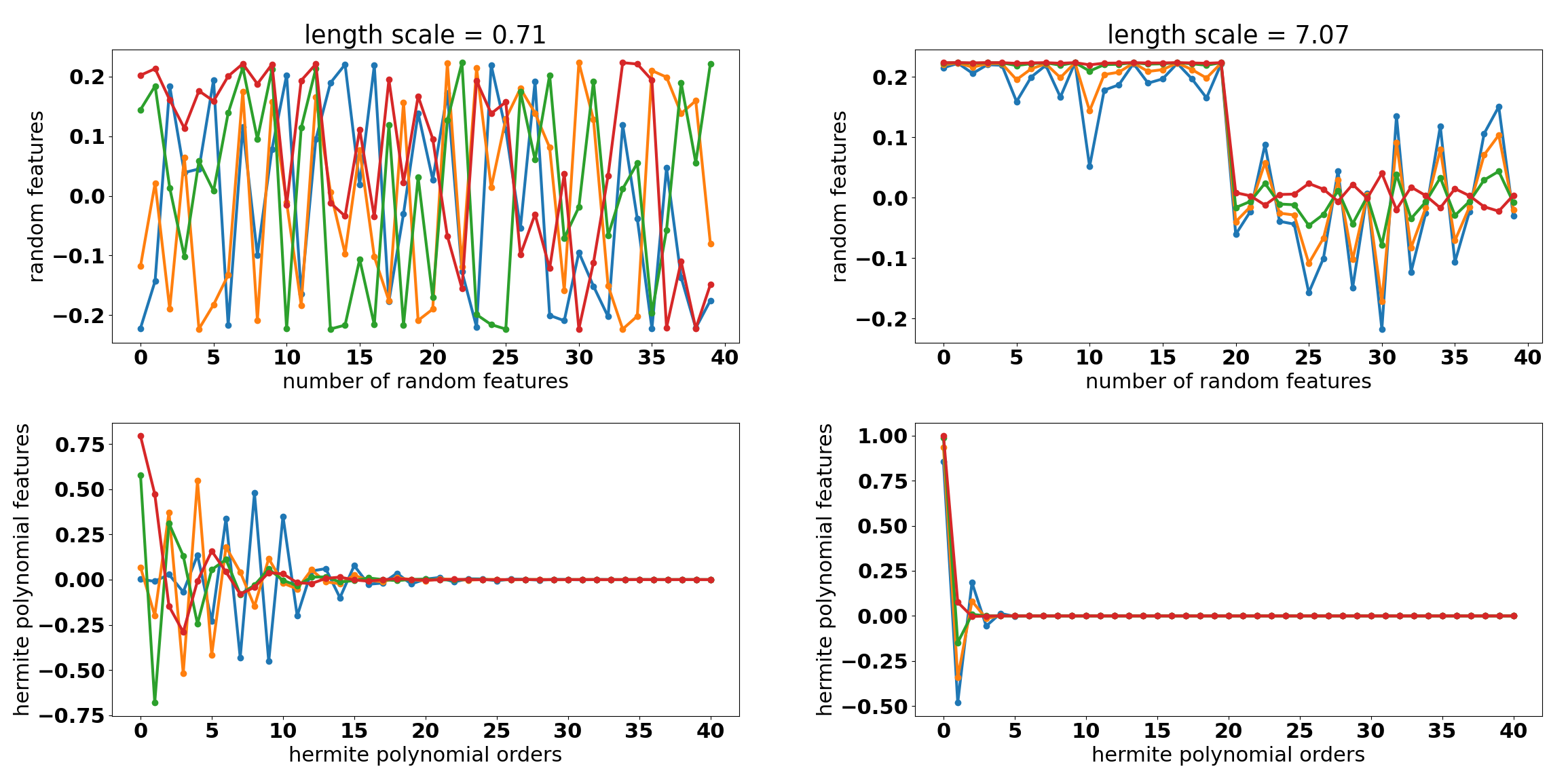}
\caption{Comparison between HP and random features at a different length scale value. Different color indicates a different datapoint, where four datapoints are drawn from $\Nrm(0,1)$.
\textbf{Left}: With length scale $l=0.71$ (relatively small compared to $1$), random features (top) at the four datapoints exhibit large variability while the Hermite polynomial features  (bottom) at those datapoints decay at around order $\leq 20$. 
\textbf{Right}: With $l=7.07$ (large compared to $1$), random features (top) exhibit less variability, while it is not clear how many features are necessary to consider. On the other hand, the Hermite polynomial features (bottom) decay fast at around order $\leq 5$ and we can make a cut-off at that order without losing much information.  
  }
  \label{fig:feature_comparison}
\end{figure}

\section{Approximation error under HP and Random Fourier features }\label{supp:prf_conv}

In the following proposition, we provide that provably our method converges with $O(\rho^{2C})$ where $\rho <1$ is the constant in the Mehler's formula, while DP-MERF has the convergence $\Omega(1/C)$, where $C$ is the number of features in each case.
\begin{prop}
Let $X$ and $Y$ be standard normal random variables. There exists a $C$-dimensional Hermite feature map $\phi_{HP}^{(C)}(\cdot)$ with the expected predictive error bounded as
\begin{equation}
    \mathbb{E}_{X, Y}\big[\big|k(X, Y)-\langle \hat{\phi}_{HP}^{(C)}(X), \hat{\phi}_{HP}^{(C)}(Y)\rangle \big|\big]\leq \frac{1}{3\sqrt{2}}(\frac{1}{3})^C. \label{eq:first_err_bound}
\end{equation} However, the expected predictive error of the random feature map $\hat{\phi}_{RF, \omega}(\cdot)$ with $C$ number of features (i.e., $\omega$ is a vector of length $C$) and the same approximating kernel is equal to
\begin{equation}
    \mathbb{E}_{\boldsymbol{\omega}, X, Y}\big[\big|k(X, Y)-\langle \hat{\phi}_{RF, \omega}(X), \hat{\phi}_{RF, \omega}(Y)\rangle\big|\big]\geq\frac{1}{8C}.\label{eq:sec_err_bound}
\end{equation}
\end{prop}
\begin{proof}

We start by proving \eqref{first_err_bound}. In this case, we write the squared error term as following:
\begin{align}
    A_{x, y} = \big|k(x, y)-\langle \hat{\phi}_{HP}^{(C)}(x), \hat{\phi}_{HP}^{(C)}(y)\rangle \big|^2
    &\overset{(a)}{=} \big|\sum_{C+1}^{\infty} \frac{\lambda_l}{\sqrt{N_l}}H_{l}(x)e^{-\frac{\rho}{1+\rho}x^2}\frac{1}{\sqrt{N_l}}H_{l}(y)e^{-\frac{\rho}{1+\rho}y^2}\big|^2\\
    &= \sum_{l, l'=C+1}^{\infty} \frac{\lambda_{l}\lambda_{l'}}{N_{l}N_{l'}} H_{l}(x)H_{l'}(x)H_{l}(y)H_{l'}(y)e^{-\frac{2\rho}{1+\rho}x^2-\frac{2\rho}{1+\rho}y^2},
\end{align}
where $(a)$ is followed by the definition of $\hat{\phi}_{HP}^{(C)}$ in \eqref{HP} and its approximation property (i.e., Mehler's formula \eqref{mehler}). Now, by setting $\rho=\frac{1}{3}$, we have
\begin{equation}
    A_{x, y} = \sum_{l, l'=C+1}^{\infty} \frac{\lambda_l \lambda_{l'}}{N_{l} N_{l'}} H_{l}(x)H_{l'}(x)H_{l}(y)H_{l'}(y)e^{-\frac{1}{2}x^2-\frac{1}{2}y^2}.
\end{equation}
Next, we average out $A_{x, y}$ for $x$s and $y$s that are drawn from a standard normal distribution as
\begin{align}
    \mathbb{E}_{X, Y\sim N(0, 1)} \big[A_{X, Y}\big] &= \int_{x, y = -\infty}^{\infty} \sum_{l, l'=C+1}^{\infty} \frac{\lambda_{l}\lambda_{l'}}{N_{l}N_{l'}} H_{l}(x)H_{l'}(x)H_{l}(y)H_{l'}(y)e^{-\frac{1}{2}x^2-\frac{1}{2}y^2} \frac{e^{-\frac{1}{2}x^2-\frac{1}{2}y^2}}{2\pi}{\rm d}x {\rm d} y\\
    &=\sum_{l, l'=C+1}^{\infty} \frac{\lambda_{l}\lambda_{l'}}{N_{l}N_{l'}} \frac{\int H_{l}(x)H_{l'}(x)e^{-x^2}{\rm d} x \int H_{l}(y)H_{l'}(y)e^{-y^2}{\rm d} y}{2\pi}\\
    &\overset{(a)}{=}\sum_{l=C+1}^{\infty} \frac{\lambda_l^2}{N_l^2}
    \frac{1}{2\pi}
    \sqrt{\pi}2^l l! \sqrt{\pi}2^{l}l!
    \overset{(b)}{=} \sum_{l=C+1}^{\infty} \frac{(2/3)^2 (1/3)^{2l}}{\frac{1}{2} 2^{2l} (l!)^2} \frac{2^{2l} (l!)^2}{2}=\frac{4}{9} \sum_{l=C+1}^{\infty} (1/3)^{2l} \\& \overset{(c)}{=}\frac{1}{2}(1/3)^{2C+2}, \label{eq:exp_conv}
\end{align}
where $(a)$ is followed by orthogonality of Hermite polynomials, $(b)$ is followed by the definition of $\lambda_{l}$ and $N_l$ in Section \ref{sec:HP}, and $(c)$ is due to the infinite Geometric series. 

As a result of \eqref{exp_conv}, the definition of $A_{x, y}$, and Jensen's inequality we have
\begin{equation}
    \mathbb{E}_{X, Y}\big[|k(X, Y)-\langle \hat{\phi}_{HP}^{(C)}(X), \hat{\phi}_{HP}^{(C)}(X)\rangle| \big]\leq \mathbb{E}^{1/2}_{X, Y}\big[A_{X, Y}\big] \leq \frac{1}{3\sqrt{2}}\big(\frac{1}{3}\big)^{C}.
\end{equation}
For bounding the expected error of random features, we expand the squared error using the definition given in \eqref{RF}:
\begin{align}
    B_{x, y, \boldsymbol{\omega}} &= \big|k(x, y)-\langle \phi_{RF, \boldsymbol{\omega}}(x), \phi_{RF, \boldsymbol{\omega}}(y)\rangle\big|^2 = \big|e^{-\frac{\rho (x-y)^2}{1-\rho^2}}-\frac{2}{C}\sum_{i=1}^{C/2}\cos \omega_i x\cos \omega_i y-\frac{2}{C}\sum_{i=1}^{C/2}\sin \omega_i x\sin \omega_i y\big|\\
    &= \underbrace{\big|e^{-\frac{\rho (x-y)^2}{1-\rho^2}}-\frac{2}{C}\sum_{i=1}^{C/2}\cos \omega_i(x-y)\big|^2}_{B_{x, y, \boldsymbol{\omega}}}.
\end{align}
Next, by setting $\rho=\frac{1}{3}$, we have 
\begin{equation}
    B_{x, y, \boldsymbol{\omega}} =e^{-\frac{3}{4}(x-y)^2}-\frac{4}{C} e^{-\frac{3}{8}(x-y)^2}\underbrace{\sum_{i=1}^{C/2} \cos \omega_i(x-y)}_{E_{1, x, y, \boldsymbol{\omega}}}+\frac{4}{C^2}\underbrace{\big(\sum_{i=1}^{C/2}\cos \omega_i (x-y)\big)^2}_{E_2, x, y, \boldsymbol{\omega}}. \label{eq:bdef}
\end{equation}
Next, we calculate the average of terms $E_{1, x, y, \boldsymbol{\omega}}$ and $E_{2, x, y, \boldsymbol{\omega}}$ over $\boldsymbol{\omega}$. 

Due to the Bochner's theorem (see Theorem 3.7 of \cite{sparseprocess}) that shows a shift-invariant positive kernel could be written in the form of Fourier transform of a density function, we have
\begin{align}
    \mathbb{E}_{\boldsymbol{\omega}}\big[E_{1, x, y, \boldsymbol{\omega}}\big] &= \mathbb{E}_{\boldsymbol{\omega}}\big[\sum_{i=1}^{C/2} \cos \omega_i (x-y)\big]\\
    &=\sum_{i=1}^{C/2}\mathbb{E}_{{\omega_i}}\big[e^{j\omega_i (x-y)}\big] = \frac{C}{2}e^{-\frac{3}{8}(x-y)^2}, \label{eq:e1}
\end{align}

Next, we obtain the average of $E_{2, x, y, \boldsymbol{\omega}}$ as following:
\begin{align}
    \mathbb{E}_{\boldsymbol{\omega}}\big[E_{2, x, y, \boldsymbol{\omega}}\big] &= \mathbb{E}_{\boldsymbol{\omega}}\big[\sum_{i, k=1}^{C/2}\cos \omega_i(x-y)\cos \omega_{k}(x-y)\big]\\
    &=\mathbb{E}_{\boldsymbol{\omega}}\big[\sum_{i, k=1}^{C/2}\frac{e^{j(\omega_i+\omega_k)(x-y)}+e^{j(\omega_i-\omega_k)(x-y)}+e^{j(-\omega_i+\omega_k)(x-y)}+e^{j(-\omega_i-\omega_k)(x-y)}}{4}\big]\\
    &\overset{(a)}{=}\sum_{i, k=1, i\neq k}^{C/2} \mathbb{E}_{\omega_i}\big[e^{j\omega_i(x-y)}\big]\mathbb{E}_{\omega_k}\big[e^{j\omega_k(x-y)}\big]+\frac{1}{2}\sum_{i=1}^{C/2}\big(\mathbb{E}_{\omega_i}\big[e^{j\omega_i(2x-2y)}\big]+1\big)\\
    &\overset{(b)}{=} \big(\frac{C^2}{4}-\frac{C}{2}\big)e^{-\frac{3}{4}(x-y)^2}+\frac{C}{4}\big(e^{-\frac{3}{4}(x-y)^2}+1\big)\\
    &=\frac{C^2}{4}e^{-\frac{3}{4}(x-y)^2}+\frac{C}{4}\big(e^{-\frac{3}{2}(x-y)^2}-2e^{-\frac{3}{4}(x-y)^2}+1\big), \label{eq:e2}
\end{align}
where $(a)$ is due to symmetry of the normal distribution of $\boldsymbol{\omega}$, and $(b)$ is followed by independence of ${\omega}_i$ and $\omega_k$ and their distribution symmetry.

Substituting \eqref{e1} and \eqref{e2} in \eqref{bdef}, and using Jensen's inequality, we have
\begin{align}
    \mathbb{E}_{X, Y\sim N(0, 1)}\mathbb{E}_{\boldsymbol{\omega}}\big[B_{x, y, \boldsymbol{\omega}}\big] &=\frac{1}{C} \mathbb{E}_{X, Y\sim N(0, 1)}\Big[\big(e^{-\frac{3}{4}(X-Y)^2}-1\big)^2\Big]\geq \frac{1}{C}\mathbb{E}^2_{X, Y}\Big[e^{-\frac{3}{4}(x-y)^2}-1\Big] \\&=\frac{1}{C}\big(\underbrace{\mathbb{E}_{X, Y\sim N(0, 1)}\big[e^{-\frac{3}{4}(X-Y)^2}\big]}_{G}-1\big)^2. \label{eq:jens}
\end{align}
To calculate $G$, we have
\begin{align}
    G &= \mathbb{E}_{X, Y\sim N(0, 1)}\big[e^{-\frac{3}{4}(X-Y)^2}\big]=\int_{x, y} \frac{e^{-\frac{3}{4}(x^2+y^2-2xy)}e^{-\frac{x^2}{2}-\frac{y^2}{2}}}{2\pi}{\rm d} x{\rm d}y\\
    &=\int_{x, y} \frac{e^{-\frac{5}{4}(x^2+y^2)+\frac{3}{2}xy}}{2\pi}{\rm d}x{\rm d}y\\
    &=\int_{x, y}\frac{e^{-\frac{5}{4}(x^2-\frac{6}{5}xy+\frac{9}{25}y^2)+\frac{9}{25}\frac{5}{4}y^2-\frac{5}{4}y^2}}{2\pi} {\rm d}x {\rm d} y\\
    &\overset{(a)}{=}\int_{y}\frac{e^{-\frac{4}{5}y^2}}{\sqrt{2\pi \frac{5}{2}}}\int_x \frac{e^{-\frac{5}{4}(x-\frac{3}{5}y)^2}}{\sqrt{2\pi \frac{2}{5}}}{\rm d} x{\rm d} y\\
    &=\int_{y}\frac{e^{-\frac{4}{5}y^2}}{\sqrt{2\pi\frac{5}{2}}}{\rm d} y = \frac{1}{2}\int_{y}\frac{e^{-\frac{4}{5}y^2}}{\sqrt{2\pi \frac{5}{8}}}\\
    &\overset{(b)}{=} \frac{1}{2}, \label{eq:gdef}
\end{align}
where $(a)$ and $(b)$ hold since for a normal distribution $f_{a, b}(x)=\frac{e^{-\frac{(x-b)^2}{2a}}}{\sqrt{2\pi a}}$, we have $\int_{x} f_{a, b}(x){\rm d} x = 1$. As a result of \eqref{jens} and \eqref{gdef} we have
\begin{equation}
    \mathbb{E}_{X, Y, \boldsymbol{\omega}}\big[B_{X, Y, \boldsymbol{\omega}}\big]\geq \frac{1}{4C}.
\end{equation}
Finally, since $0\leq B_{x, y, \boldsymbol{\omega}}\leq 4$, we have
\begin{align}
    \frac{1}{16 C}\leq \mathbb{E}_{X, Y, \boldsymbol{\omega}}\big[\frac{B_{X, Y, \boldsymbol{\omega}}}{4}\big]\leq \mathbb{E}_{X, Y, \boldsymbol{\omega}}\big[\frac{|B_{X, Y, \boldsymbol{\omega}}|^{1/2}}{2}\big]=\frac{1}{2}\mathbb{E}_{X, Y, \boldsymbol{\omega}}\big[\big|k(X, Y)-\langle \phi_{RF, \boldsymbol{\omega}}(X),\phi_{RF, \boldsymbol{\omega}}(Y)\rangle \big|\big],
\end{align}
which proves \eqref{sec_err_bound}.
\end{proof}

\section{Mercer's theorem and the generalized Hermite polynomials}\label{supp:generalized}

We first review Mercer's theorem, which is a fundamental theorem on how can we find the approximation of a kernel via finite-dimensional feature maps.

\begin{thm}[\cite{smola1998learning} Theorem 2.10 and Proposition 2.11 ]  \label{thm:marcer}
Suppose $k\in L_{\infty}(\mathcal{X}^2)$, is a symmetric real-valued function, for a non-empty set $\mathcal{X}$, such that the integral operator $T_kf(x)=\int_{\mathcal{X}} k(x, x') f(x') \partial \mu(x')$ is positive definite. Let $\psi_j\in L_2(\mathcal{X})$ be the normalized orthogonal eigenfunctions of $T_k$ associated with the eigenvalues $\lambda_j>0$, sorted in non-increasing order, then 
\begin{enumerate}
    \item $(\lambda_j)_j \in \ell_1$,
    \item $k(x, x')=\sum_{j=1}^{N_{\mathcal{H}}} \lambda_j \psi_j(x) \psi_j (x')$ holds for almost all $(x, x')$. Either $N_{\mathcal{H}}\in \mathbb{N} $, or $N_{\mathcal{H}} =\infty$; in the latter case, the series converge absolutely and uniformly for almost all $(x, x')$.
\end{enumerate}

Furthermore, for every $\epsilon>0$, there exists $n$ such that

\begin{equation}
|k(x, x')- \sum_{j=1}^{n} \lambda_j \psi_j(x) \psi_j(x')|<\epsilon,   
\end{equation}

for almost all $x, x'\in \mathcal{X}$. 
\end{thm}

This theorem states that one can define a feature map 
\begin{equation}
\Phi_n (x ) = \big[\sqrt{\lambda_1} \psi_1(x), \ldots, \sqrt{\lambda_n} \psi_n(x)\big]^T \label{eq:appr_k}
\end{equation}

such that the Euclidean inner product  $\langle \Phi (x), \Phi(x')\rangle$ approximates $k(x, x')$ up to an arbitrarily small factor $\epsilon$. 

By means of uniform convergence in Mercer's theorem, we can prove the convergence of the approximated MMD using the following lemma.

\begin{lem} \label{lem:unif_mmd}
Let $\mathcal{H}$ be an RKHS that is generated by the kernel $k(\cdot, \cdot)$, and let $\widehat{\mathcal{H}}_n$ be an RKHS with a kernel $k_n(\vx, \vy)$ that can uniformly approximate $k(\vx, \vy)$. Then, for a positive real value $\epsilon$, there exists $n$, such that for every pair of distributions $P, Q$, we have
\begin{equation}
\big|{\rm MMD}^2_{\mathcal{H}}(P, Q)-{\rm MMD}^2_{\widehat{\mathcal{H}}_n}(P, Q)\big|< \epsilon.
\end{equation}

\end{lem}
\begin{proof}
Firstly, using Theorem \ref{thm:marcer}, we can find $n$ such that $\big|k(x, y)-\langle \Phi_n(x), \Phi_n(y)\rangle\big|<\frac{\epsilon}{4}$.  We define the RKHS $\widehat{H}_n$ as the space of functions spanned by $\Phi_n(\cdot)$. Next, we rewrite ${{\rm MMD}^2_{\mathcal{H}}(P, Q)-{\rm MMD}^2_{\widehat{\mathcal{H}}_n}(P, Q)}$, using the definition of MMD in Section \ref{sec:mmd}, as 

\begin{align}
&\mathrm{MMD}^2_{\mathcal{H}}(P,Q) - \mathrm{MMD}^2_{\widehat{\mathcal{H}}_n}(P, Q) \nonumber \\
& = \mathbb{E}_{x, x'\sim P} \big[k(x, x')\big] +\mathbb{E}_{y, y'\sim Q} \big[k(y, y')\big]-2 \mathbb{E}_{x\sim P, y\sim Q} \big[k(x, y)\big] \nonumber \\
&- \mathbb{E}_{x, x'\sim P} \big[\langle \Phi_n(x), \Phi_n(x')\rangle\big] +\mathbb{E}_{y, y'\sim Q} \big[\langle \Phi_n(y), \Phi_n(y')\rangle\big]-2 \mathbb{E}_{x\sim P, y\sim Q} \big[\langle \Phi_n(x), \Phi_n(y)\rangle\big] 
\end{align}

Therefore, we can bound $\big|{\rm MMD}^2_{\mathcal{H}}(P, Q)-{\rm MMD}^2_{\widehat{\mathcal{H}}_n}(P, Q)\big|$ as 

\begin{equation}
\big|{\rm MMD}^2_{\mathcal{H}}(P, Q)-{\rm MMD}^2_{\widehat{\mathcal{H}}_n}(P, Q)\big| \overset{(a)}{\leq}  \bigg|\mathbb{E}_{x, x'\sim P}\big[k(x, x')\big]-\mathbb{E}_{x, x'\sim P}\Big[\big\langle \Phi_n(x), \Phi_n(x')\big\rangle\Big]\bigg|  \nonumber
\end{equation}

\begin{equation}
+ \bigg|\mathbb{E}_{y, y'\sim Q}\big[k(y, y')\big]-\mathbb{E}_{y, y'\sim P}\Big[\big\langle \Phi_n(y), \Phi_n(y')\big\rangle\Big]\bigg|
+ 2\bigg|\mathbb{E}_{x, y\sim P, Q}\big[k(x, y)\big]-\mathbb{E}_{x, y\sim P, Q}\Big[\big\langle \Phi_n(x), \Phi_n(y)\big\rangle\Big]\bigg|  \nonumber 
\end{equation}

\begin{align}
&\overset{(b)}{\leq}  \mathbb{E}_{x, x'\sim P}\bigg[\Big|k(x, x')-\big\langle \Phi_n(x), \Phi_n(x')\big\rangle\Big|\bigg] + \mathbb{E}_{y, y'\sim Q}\bigg[\Big|k(y, y')-\big\langle \Phi_n(y), \Phi_n(y')\rangle\Big|\bigg] \nonumber\\&\phantom{\overset{(b)}{\leq}  \mathbb{E}_{x, x'\sim P}\bigg[\Big|k(x, x')-\big\langle \Phi_n(x), \Phi_n(x')\big\rangle\Big|\bigg] +}+2\mathbb{E}_{x, y\sim P, Q}\bigg[\Big|k(x, y)-\big\langle \Phi_n(x), \Phi_n(y)\big\rangle\Big|\bigg] \nonumber \\
&\overset{(c)}{\leq} \mathbb{E}_{x, x'\sim P}\big[\frac{\epsilon}{4}\big]+\mathbb{E}_{y, y'\sim Q}\big[\frac{\epsilon}{4}\big]+2\mathbb{E}_{x, y\sim P, Q} \big[\frac{\epsilon}{4}\big]=\epsilon
\end{align}
where $(a)$ holds because of triangle inequality, $(b)$ is followed by Tonelli's theorem and Jensen's inequality for absolute value function, and $(c)$ is correct because of the choice of $n$ as mentioned earlier in the proof.
\end{proof}

As a result of the above theorems, we can approximate the MMD in RKHS $\mathcal{H}_{k}$ for a kernel $k(\cdot, \cdot)$ via MMD in RKHS $\widehat{\mathcal{H}}_n\subseteq \mathbb{R}^n$ that is spanned by the first $n$ eigenfunctions weighted by square roots of eigenvalues of the kernel $k(\cdot, \cdot)$. Therefore, in the following section, we focus on finding the eigenfunctions/eigenvalues of a multivariate Gaussian kernel.

\subsection{Generalized Mehler's approximation}
As we have already seen in \eqref{mehler}, Mehler's theorem provides us with an approximation of a one-dimensional Gaussian kernel in terms of Hermite polynomials. To generalize Mehler's theorem to a uniform covergence regime (that enables us to approximate MMD via such feature maps as shown in Lemma \ref{lem:unif_mmd}), and for a multivariate Gaussian kernel, we make use of the following theorem.

\begin{thm}[\cite{slepian1972symmetrized}, Section  6] \label{thm:slepian}
Let the joint Gaussian density kernel $k(\vx, \vy, C):\mathbb{R}^{n}\times \mathbb{R}^{n}\to \mathbb{R}$ be

\begin{equation}
 k(\vx, \vy, C) = \frac{1}{(2\pi)^{n}|C|^{1/2}} \exp \Big(-\frac{1}{2}[\vx, -\vy] C^{-1} [\vx, -\vy]^{T}\Big) 
\end{equation}

where $C$ is a positive-definite matrix as
\begin{equation}
C = \left[ \begin{array}{c c}
    C_{11} & C_{12} \\
    C_{12}^T & C_{22}
\end{array}\right],
\end{equation}

in which $C_{ij}\in \mathbb{R}^{n\times n}$ for $i, j\in \{1, 2\}$, and $C_{11}=C_{22}$. Further, let the integral operator be defined with respect to a measure with density
\begin{equation}
w(\vx) = \frac{1}{\int k(\vx, \vy, C) \partial \vy}.
\end{equation}

Then, the orthonormal eigenfunctions and eigenvalues for such kernel are
\begin{equation}
\psi_{\vk}(\vx) = \sum_{\vl: \|\vl\|_1=\|\vk\|_1}  \big(\sigma_{\|\vk\|_1}(P)^{-1}\big)_{\vk \vl} \frac{\varphi_{\vl}(\vx; C_{11})}{\sqrt{\prod_{i=1}^n l_i!}}, \label{eq:eigen_herm}
\end{equation}

and 
\begin{equation}
\lambda_{\vk} = \prod_{i=1}^n \ve_{i}^{\vk_i/2}. \label{eq:eigen_val_herm}
\end{equation}
Here, $\sigma_{p}(A)$ is symmetrized Kronecker power of a matrix $A$, defined as

\begin{equation}
\big(\sigma_{\|\vk\|_1}(A)\big)_{\vk \vl} = \sqrt{\prod_{i=1}^n k_i! l_i!} \sum_{M\in \mathbb{R}^{n\times n}: M \mathds{1}_n = \vk, \mathds{1}_n^T M = \vl} \frac{\prod_{ij} A_{ij}^{M_ij}}{\prod_{ij} M_{ij}!}, \label{eq:kron_pow}
\end{equation}

for two $n$-dimensional vectors $\vk$ and $\vl$ with $\|\vk\|_1=\|\vl\|_1$, the vector $\ve$ (the matrix $P$) is formed by eigenvalues (eigenvectors) of $C_{11}^{-1}C_{12}$, and $\varphi_{\vl}(\vx, A)$ is generalized Hermite functions defined as

\begin{equation}
\varphi_{\vl} (\vx, A) = \frac{1}{(2\pi)^{n/2} |A|^{1/2}} \frac{\partial^{\|\vl\|_1}}{\partial{x_1}^{l_1}\ldots \partial{x_n}^{l_n}}\exp \big(-\frac{1}{2} \vx^ T A^{-1} \vx\big).
\end{equation}

\end{thm}

The above theorem provides us with eigenfunctions/eigenvalues of a joint Gaussian density function. We utilize this theorem to approximate Mahalanobis kernels (i.e., a generalization of Gaussian radial basis kernels where $A=c I_n$) via Hermite polynomials as follow.

\begin{prop}\label{prop:mahal}
A Mahalanobis kernel $k(\vx, \vy, A):\mathbb{R}^{D}\times \mathbb{R}^{D}\to \mathbb{R}$ defined as 
$$k(\vx, \vy, A)= \exp\big(-(\vx-\vy)A(\vx-\vy)^T\big)$$ can be uniformly approximated as

\begin{equation}
k(\vx, \vy, A)\simeq \Big\langle \Phi_N\Big(\sqrt{\frac{\alpha^2-1}{\alpha}}\sqrt{A}\vx\Big), \Phi_N\Big(\sqrt{\frac{\alpha^2-1}{\alpha}}\sqrt{A}\vy\Big)\Big\rangle,
\end{equation}
where $\Phi(\vx)\in N^{D}$ is defined as a tensor product
\begin{equation}
\Phi_N(\vx) = \bigotimes_{i=1}^n [\phi_{k_i}(x_i)]_{k_i=1}^N, \label{eq:tensor_fmap}
\end{equation}

where

\begin{equation}
\phi_{k_i}(x_i)= \left( \frac{(\alpha^2-1)\alpha^{-k_i}}{\alpha^{2} k_i!} \right)^{1/4} \exp \left( \frac{-x_i^2}{\alpha+1} \right) H_{k_i}(x_i)
\end{equation}

\end{prop}

\begin{remark}
 \label{rem: unif_gen}
Using Proposition \ref{prop:mahal} and Lemma \ref{lem:unif_mmd}, we can show that the MMD based on the tensor feature map in \eqref{tensor_fmap} and between any two distributions approximates the real MMD based on Gaussian kernel with Mahalanobis norm.  
\end{remark}

\begin{proof}[Proof of Proposition \ref{prop:mahal}]
Let $C=\left[\begin{array}{cc}
      \frac{1}{2}I_n& \frac{1}{2\alpha}I_n  \\
     \frac{1}{2\alpha}I_n&  \frac{1}{2}I_n 
\end{array}\right]$, or equivalently $C^{-1}=\left[\begin{array}{cc}
     \frac{2\alpha^2}{\alpha^2-1} I_n& -\frac{2\alpha}{\alpha^2-1} I_n  \\
     -\frac{2\alpha}{{\alpha^2-1}} I_n& \frac{2\alpha^2}{\alpha^2-1} I_n 
\end{array}\right]$, for $\alpha\in [1, \infty)$.

Since $C$ is positive-definite, we can define a Gaussian density kernel as
\begin{equation}
k(\vx, \vy, C)= \frac{1}{(\frac{\pi\sqrt{\alpha^2-1}}{2\alpha})^n} \exp \big(-\frac{\alpha^2}{\alpha^2-1}\|\vx\|^2-\frac{\alpha^2}{\alpha^2-1}\|\vy\|^2+\frac{2\alpha}{\alpha^2-1} \vy  \cdot \vx^T\big).
\end{equation}

Moreover, we can calculate the integration over all values of~$\vy$ as
\begin{equation}
\int k(\vx, \vy, C)\partial \vy =\int \frac{\exp\big(-\|\vx\|^2\big)}{(\frac{\pi\sqrt{\alpha^2-1}}{2\alpha})^n} \exp\big(-\frac{\|\alpha\vy-{\vx }\|^2}{(\alpha^2-1)}\big)\partial \vy 
 =\frac{\exp\big(-\|\vx\|^2\big)}{({\pi})^{n/2}}.\label{eq:k_int}
\end{equation}

Next, by setting $w(\vx)=\frac{1}{\int k(\vx, \vy, C)\partial \vy}$ and using Theorem \ref{thm:slepian}, we have
\begin{equation}
\int \frac{1}{(\pi\frac{\alpha^2-1}{\alpha^2})^{n/2}} \psi_{\vk} (\vx) \exp\big(-\frac{\|\alpha\vy-{\vx }\|^2}{\alpha^2-1}\big)\partial \vx = \lambda_{\vk}\psi_{\vk} (\vy). \label{eq:norm_herm_int}
\end{equation}

Now to find the eigenfunctions of the Gaussian kernel $k'(\vx, \vy)= \exp\big(-\frac{\alpha\|\vx~-~\vy\|^2}{(\alpha^2-1)}\big)$, we let $\psi_{\vk}'(\vx)={\psi_{\vk}(\vx)} \exp\big(\frac{\alpha}{\alpha+1}\|\vx\|^2\big)$ and let the weight function be $w'(\vx)=~{(\pi)^{n/2}}\exp\big(-~\frac{(\alpha-1)}{\alpha+1}\|\vx\|^2\big)$. As a result of such assumptions, we see that
\begin{align}
\int \psi_{\vk}'(\vx)&k'(\vx, \vy)w'(\vx)\partial \vx \nonumber\\&=\int (\pi)^{n/2}{\psi_{\vk}(\vx)}{} \exp\big(-\frac{1}{\alpha^2-1}\|\vx\|^2-\frac{\alpha}{\alpha^2-1}\|\vy\|^2+\frac{2\alpha}{\alpha^2-1} \vx\cdot \vy^T\big)\partial\vx\\
&=(\pi)^{n/2}{\exp \big(\frac{\alpha}{\alpha+1}\|\vy\|^2\big)} \int \psi_{\vk} (\vx) \exp\big(-\frac{\|\alpha\vy-{\vx}\|^2}{\alpha^2-1}\big)\partial \vx\\
&\overset{(a)}{=} (\pi)^{n/2}{\exp \big(\frac{\alpha}{\alpha+1}\|\vy\|^2\big)}  \sqrt{\lambda_{\vk}}\psi_{\vk}(\vy)\Big(\frac{\pi(\alpha^2-1)}{\alpha^2}\Big)^{n/2}\\
&\overset{(b)}{=} (\pi)^{n}{\Big(\frac{\alpha^2-1}{\alpha^2}\Big)^{n/2} \lambda_{\vk}\psi_{\vk}'(\vy)}, \label{eq:lambda_psi_prime}
\end{align}

where $(a)$ holds because of \eqref{norm_herm_int}, and $(b)$ is followed by the definition of $\psi_{\vk}'(\vy)$. 
As a result, $\psi_{\vk}'(\vx)$ is an eigenfunction of the integral operator with kernel $k'(\vx, \vy)$ and with weight function $w'(\vx)$.  

Equation \eqref{lambda_psi_prime} shows that the eigenvalue of $k'(\vx, \vy)$ corresponding to $\psi_{\vk}(\vx)$  is as
\begin{equation}
\lambda'_{\vk} = (\pi)^{n}\Big(\frac{\alpha^2-1}{\alpha^2}\Big)^{n/2} \lambda_{\vk}
\end{equation}

Now we show that such eigenfunctions are orthonormal. Deploying the idea in \eqref{lambda_psi_prime}, for two eigenfunctions $\psi'_{\vk}(\cdot)$ and $\psi'_{\vl}(\cdot)$ for fixed vectors $\vk, \vl \in \mathbb{N}^{n}$, we have
\begin{equation}
 \int \psi_{\vk}'(\vy)\psi_{\vl}'(\vy)w'(\vy)\partial \vy
\overset{(a)}{=} \int\psi_{\vk}(\vy)\psi_{\vl}(\vy)\frac{(\pi)^{n/2}}{\exp\big(-\|\vx\|^2\big)}\partial \vy\\
\overset{(b)}{=} \int \psi_{\vk}(\vy)\psi_{\vl}(\vy)w(\vy)\overset{(c)}{=}\delta[\vl-\vk],
\end{equation}

where $(a)$ is followed by the definition of eigenfunctions $\psi'_{\vk}(\cdot), \psi_{\vl}'(\cdot)$ and the definition of weight function $w'(\vx)$, $(b)$ is due to the definition of $w(\vx)$ and \eqref{k_int}, and $(c)$ holds because  of orthonormality of $\psi_{\vk}$s as a result of Theorem \ref{thm:slepian}. 

Further, in this case we have $C_{11}^{-1}C_{12} = \frac{1}{\alpha} I_{n}$, or equivalently $P=I_n$ and  $\ve = \frac{1}{\alpha} \mathds{1}_n$. Hence, firstly using \eqref{eigen_val_herm}, one can see that
\begin{equation}
\lambda_{\vk} = \alpha^{-\|\vk\|/2}.
\end{equation}
Secondly, in finding symmetrized Kronecker power $\sigma_{\|k\|_1} (P)$ in \eqref{kron_pow}, for non-diagonal matrices $M$, the term $\prod_{ij} P_{ij}^{M_{ij}}=0$. Further, for a diagonal matrix $M$, we have $M \mathds{1}_n = \mathds{1}_{n} M$. This induces the fact that 

\begin{equation}
\sigma_{\|k\|_1} (P) = \left\{\begin{array}{c c}
    0 &  \vk\neq \vl, \\
    1 &  \vk =\vl
\end{array} \right. .
\end{equation}
This shows that 
\begin{equation}
\psi_{\vl}(\vx)= \frac{\varphi_{\vl}(\vx)}{\sqrt{\prod_{i=1}^n l_i!}}.
\end{equation}

To find the formulation of eigenfunction $\psi_k(\vx)$, we can rewrite the term $\varphi_{\vl}(\vx, C_{11})$ in \eqref{eigen_herm} for $C_{11}=\frac{1}{2} I_n$ as 
\begin{equation}
\varphi_{\vl} (\vx,  I)= \frac{1}{({\pi})^{n/2} } \frac{\partial^{\|\vl\|_1}}{\partial{x_1}^{l_1}\ldots \partial{x_n}^{l_n}} \exp \Big(-\sum_{i=1}^n x_i^2\Big). \label{eq:phi_x_alpha_I}
\end{equation}

We note that the exponential function can be written as the product of functions that are only dependent on one variable $x_i$ for $i\in [n]$. Hence, we can rephrase \eqref{phi_x_alpha_I} as a product of the derivative of each function as
\begin{equation}
\varphi_{\vl} (\vx,  I) = \prod_{i=1}^n \frac{1}{\sqrt{{\pi}}} \frac{\partial^{l_i}}{\partial^{l_i} x_i} \exp \big(- x_i^2\big).
\end{equation}
As a result of this equation and the definition of Hermite functions in one dimension, we have
\begin{equation}
\varphi_{\vl} (\vx,  I) = \frac{\exp(-\|\vx\|^2)}{({\pi})^{n/2}}\prod_{i=1}^n  H_{l_i}(x_i)
\end{equation}
Hence, we can calculate $\psi'_{\vk}(\vx)$ as
\begin{equation}
\psi'_{\vk}(\vx)= \frac{1}{\sqrt{(\pi)^n  \prod_{i=1}^n k_i!}} \exp(\frac{-\|\vx\|^2}{\alpha+1}) \prod_{i=1}^n H_{k_i}(x_i).
\end{equation}

Using above discussion, we see that $\vk$-th element $[\Phi_N(\vx)]_{\vk}$ of the tensor $\Phi_N(x)$, which is defined in the proposition statement, is equal to
\begin{equation}
[\Phi_N(\vx)]_{\vk}=\sqrt{\lambda'_{\vk}} \psi'_{\vk}(\vx).
\end{equation}

This fact and Theorem \ref{thm:marcer} concludes that we can uniformly approximate $k'(\vx, \vy)$ as
\begin{equation}
k'(\vx, \vy) = \langle \Phi_N(\vx), \Phi_N(\vy)\rangle. 
\end{equation}

Further, for any positive-definite matrix $A$, since the singular values of $\sqrt{\frac{\alpha^2-1}{\alpha}}\sqrt{A}$ are bounded, one can uniformly approximate $k''(\vx, \vy):=\exp\big(-(\vx-\vy) A (\vx-\vy)^T\big)=k'\Big(\sqrt{\frac{\alpha^2-1}{\alpha}}\sqrt{A} \vx, \sqrt{\frac{\alpha^2-1}{\alpha}}\sqrt{A} \vy\Big)$ as
\begin{equation}
k''(\vx, \vy) \simeq \Big\langle \Phi_N\Big(\sqrt{\frac{\alpha^2-1}{\alpha}}\sqrt{A}\vx\Big), \Phi_N\Big(\sqrt{\frac{\alpha^2-1}{\alpha}}\sqrt{A}\vy\Big)\Big\rangle
\end{equation}

\end{proof}

\section{Sum-kernel upper-bound} \label{supp:sum_kernel}
Instead of using Generalized Hermite mean embedding which takes a huge amount of memory, one could use an upper bound to the joint Gaussian kernel. We use the inequality of arithmetic and geometric means to prove that.

\begin{align}
k(\vx, \vy)= \exp\Big(-\frac{1}{2l^2} (\vx-\vy)^T(\vx-\vy)\Big) &= \exp ( -\frac{1}{2l^2}\sum_{d=1}^D (x_d-y_d)^2\Big)\label{eq:kern}\\
&= \prod_{d=1}^D \exp \Big(-\frac{1}{2l^2}(x_d-y_d)^2\Big)\\
&\overset{(a)}{\leq} \frac{1}{D}\sum_{d=1}^D \exp\Big(-\frac{D}{2l^2}(x_d-y_d)^2\Big) \\
&= \frac{1}{D} \sum_{d=1}^D k_{X_d}(x_d, y_d),
\end{align}
where $(a)$ holds due to inequality of arithmetic and geometric means (AM-GM), and $k_{X_d}(\cdot, \cdot)$ is defined as
\begin{equation}
k_{X_d}(x_d, y_d):=\exp\Big(-\frac{D}{2l^2}(x_d-y_d)^2\Big).\label{eq:kern_marg}
\end{equation}

Next, we approximate such kernel via an inner-product of the feature maps 
\begin{align}\label{eq:feature_map_kern}
    \vphi_C(\vx) &=     \begin{bmatrix} 
    \vphi^{(C)}_{HP,1}(x_{1})/\sqrt{D}   \\
    \vphi^{(C)}_{HP,2}(x_{2})/\sqrt{D}   \\
    \vdots \\
     \vphi^{(C)}_{HP,D}(x_{D})/\sqrt{D} 
    \end{bmatrix}   \in \mathbb{R}^{( (C+1) \cdot D) \times 1}.
\end{align}

Although such feature maps are not designed to catch correlation among dimensions, they provide us with a guarantee on marginal distributions as follows.

\begin{lem} \label{lem:marginal_mmd}
Define $k_{X_i}(\cdot, \cdot)$ as in \eqref{kern_marg} and define $\phi_{C}(\vx)$ as in \eqref{feature_map_kern}. 
For $\epsilon \in \mathbb{R}^{+}$, there exists $N$ such that for $C\geq N$ we have
\begin{itemize}
    \item $\big\| \mathbb{E}_{\vx\sim P}\big[\phi_{C}(\vx)\big]- \mathbb{E}_{\vy\sim Q}\big[\phi_{C}(\vy)\big] \big\|_2\leq \epsilon \Rightarrow {\rm MMD}_{k_{X_i}}(P_i, Q_i)\leq \sqrt{D+1}\epsilon$ for every $i\in \{1, \ldots, D\}$, and
    \item ${\rm MMD}_{k_{X_i}}(P_i, Q_i)\leq \epsilon $ for every $i\in\{1, \ldots, D\}$ $\Rightarrow \big\| \mathbb{E}_{\vx\sim P}\big[\phi_{C}(\vx)\big]~-~ \mathbb{E}_{\vy\sim Q}\big[\phi_{C}(\vy)\big] \big\|\leq~\sqrt{2}\epsilon$,
\end{itemize}
where $P_i$ and $Q_i$ are marginal probability distributions corresponding to $P$ and $Q$, respectively.
\end{lem}
\begin{proof}
Since $\vphi_{HP i}^{(C)}(x_i)$ has the certain form as in Theorem \ref{thm:marcer}, then Lemma \ref{lem:unif_mmd} shows that we can use such feature maps to uniformly approximate the MMD in an RKHS based on the kernel $k_i(x_i, y_i)= \exp\big(-\frac{1}{2l^2}(x_i-y_i)^2\big)$. As a result, there exists $N$ such that for $C\geq N$, we have
\begin{equation}
\Big|\big\|\mathbb{E}_{x_i\sim P_i}\big[\phi^{(C)}_{HP, i}(x_i)\big]-\mathbb{E}_{y_i\sim Q_i}\big[\phi^{(C)}_{HP,i}(y_i)\big] \big\|_2^2-{\rm MMD}_{k_{X_i}}^2(P_i, Q_i)\Big|\leq D\epsilon^2.\label{eq:marg_uniform_appr}
\end{equation}

Now we prove the first part. Knowing
\begin{equation}
\big\|\mathbb{E}_{\vx\sim P}\big[\phi_{C}(\vx)\big]-\mathbb{E}_{\vy\sim Q}\big[\phi_{C}(\vy)\big] \big\|_2\leq \epsilon,
\end{equation}
and by the definition of $\phi_C(\cdot)$, we deduce that
\begin{equation}
\big\|\mathbb{E}_{x_i\sim P_i}\big[\phi^{(C)}_{HP, i}(x_i)\big]-\mathbb{E}_{y_i\sim Q_i}\big[\phi^{(C)}_{HP,i}(y_i)\big] \big\|_2^2\leq \epsilon^2.
\end{equation}
Using this and \eqref{marg_uniform_appr} we can prove the first part.

Inversely, by setting ${\rm MMD}_{k_{X_i}}(P_i, Q_i)\leq \epsilon$ and \eqref{marg_uniform_appr}, one sees that
\begin{equation}
\big\|\mathbb{E}_{x_i\sim P_i}\big[\phi^{(C)}_{HP, i}(x_i)\big]-\mathbb{E}_{y_i\sim Q_i}\big[\phi^{(C)}_{HP,i}(y_i)\big] \big\|_2\leq \sqrt{2}\epsilon.
\end{equation}
This coupled with the definition of $\Phi_C$ completes the second part of lemma.
\end{proof}

\section{$\phi$ Recursion}\label{supp:recursion}
\begin{align}
\phi_{k+1} (x) 
&= \left((1+\rho)(1-\rho)\right)^{\frac{1}{4}} \frac{\rho^{\frac{k+1}{2}}}{\sqrt{2^{k+1} (k+1)!}} H_{k+1}(x) \exp\left(-\frac{\rho}{\rho+1}x^2\right), \quad \mbox{by definition} \nonumber
\\
&= \left((1+\rho)(1-\rho)\right)^{\frac{1}{4}} \frac{\rho^{\frac{k+1}{2}}}{\sqrt{2^{k+1} (k+1)!}} \left[ 2x H_k(x)  - 2k  H_{k-1}(x) \right] \exp\left(-\frac{\rho}{\rho+1}x^2\right), \nonumber
\\
&= \frac{\sqrt{\rho}}{\sqrt{2(k+1)}} 2x \phi_k(x) - \frac{\rho}{\sqrt{k(k+1)}} k \phi_{k-1}(x). 
\end{align} 

\section{Sensitivity of mean embeddings (MEs)}\label{supp:sensitivity}
\subsection{Sensitivity of ME under the sum kernel}
Here we derive the sensitivity of the mean embedding corresponding to the sum kernel.
\begin{align}
    S_{\widehat\vmu^s_P} &= \max_{\Dat, \Dat'} \|\widehat\vmu^s_P(\Dat) - \widehat\vmu^s_P(\Dat') \|_F = \max_{\Dat, \Dat'}\|  \tfrac{1}{m} \sum_{i=1}^m \vh_s(\vx_i) \vf(\vy_i)^T - \tfrac{1}{m} \sum_{i=1}^m \vh_s(\vx'_i) \vf(\vy'_i)^T   \|_F \nonumber
\end{align} where $\|\cdot\|_{F}$ represents the Frobenius norm. 
%
Since $\Dat$ and $\Dat'$ are neighbouring, then $m-1$ of the summands on each side cancel and we are left with the only distinct datapoints, which we denote as $(\vx, \vy)$ and $(\vx', \vy')$. We then apply the triangle inequality and the definition of $\vf$. As $\vy$ is a one-hot vector, all but one column of $\vh_s(\vx) \vf(\vy)\trp$ are 0, so we omit them in the next step:
\begin{align}
    S_{\vmu^s_P}
    &= \max_{(\vx, \vy), (\vx', \vy')}\|  \tfrac{1}{m} \vh_s(\vx) \vf(\vy)^T - \tfrac{1}{m} \vh_s(\vx') \vf(\vy')^T   \|_F \nonumber \\
    &\leq  \max_{(\vx, \vy)}  \tfrac{2}{m} \| \vh_s(\vx) \vf(\vy)^T \|_F=  \max_{\vx}  \tfrac{2}{m} \| \vh_s(\vx) \|_2. \label{eq:sens}
\end{align}

We recall the definition of the feature map given in \eqref{feature_map},
\begin{align}
    \| \vh_s(\vx) \|_2 
    &= \frac{1}{\sqrt{D}}\left(\sum_{d=1}^D \|\vphi_{HP, d}^{(C)}(x_d)\|_2^2 \right)^{\frac{1}{2}}. \label{eq:norm_phi}
\end{align}    
To bound $\|\vh_s(\vx)\|_2$, we first prove that $\|\vphi^{(C)}_{HP, d}(x_d)\|_2^2\leq 1$. 
Using Mehler's formula (see \eqref{mehler}), and by plugging in $y=x_d$, one can show that
\begin{equation}
1 = \exp \Big(-\frac{\rho}{1-\rho^2}(x_d-x_d)^2\Big) =\sum_{c=0}^{\infty} \lambda_c f_c(x_d)^2.
\end{equation}

Using this, we rewrite the infinite sum in terms of the $C$th-order approximation and the rest of summands to show that
\begin{equation}
1=\sum_{c=0}^{\infty} \lambda_c f^2_c(x_d) \overset{(a)}{=} \|\vphi_{HP, d}^{(C)}(x_d)\|_2^2+ \sum_{c=C+1}^{\infty} \lambda_c f^2_c(x)\overset{(b)}{\geq} \|\vphi_{HP, d}^{(C)}(x_d)\|_2^2, \label{eq:bound_phi_norm}
\end{equation}
where $(a)$ holds because of the definition of $\vphi^{(C)}_{HP, d}(x_d)$ in \eqref{HP}:  $\|\vphi^{(C)}_{HP, d}(x_d)\|_2^2 = \sum_{c=0}^C \lambda_c f^2_c(x_d)$, and $(b)$ holds, because $\lambda_c$ and $f_c^2(x)$ are non-negative scalars.

Finally, deploying  \eqref{sens}, \eqref{norm_phi}, and \eqref{bound_phi_norm}, we bound the sensitivity as 
\begin{align}
    S_{\vmu_P}
    &\leq \max_{\vx}  \tfrac{2}{m} \| \vh_s(\vx) \|_2 \leq \tfrac{2}{m \sqrt{D}} \sqrt{D} =   \tfrac{2}{m}. 
\end{align}

\subsection{Sensitivity of ME under the product kernel}
Similarly, we derive the sensitivity of the mean embedding corresponding to the product kernel.
\begin{align}
    S_{\widehat\vmu^p_P} &= \max_{\Dat, \Dat'} \|\widehat\vmu^p_P(\Dat) - \widehat\vmu^p_P(\Dat') \|_F 
    \leq \max_{\vx}  \tfrac{2}{m} \| \vh_p(\vx^{D_{prod}}) \|_2 \nonumber 
\end{align}
Given the definition in \eqref{full_feature_map_prod}, the L2 norm is given by 
\begin{align}
    \tfrac{2}{m}\| \vh_p(\vx^{D_{prod}}) \|_2 &=\tfrac{2}{m} \prod_{d=1}^{D_{prod}} \|\vphi_{HP}^{(C)}(x_d)\|_2, \\
    &\leq \tfrac{2}{m}
\end{align} where the last line is due to \eqref{bound_phi_norm}.

\section{Descriptions on the tabular datasets}\label{supp:app_results}

In this section we give more detailed information about the tabular datasets used in our experiments. Unless otherwise stated, the datasets were obtained from the UCI machine learning repository \cite{Dua:2019}.

\textbf{Adult}

Adult dataset contains personal attributes like age, gender, education, marital status or race from the different dataset participants and their respective income as the label (binarized by a threshold set to 50K). The dataset is publicly available at the UCI machine learning repository at the following link:  \url{https://archive.ics.uci.edu/ml/datasets/adult}.

\textbf{Census}

The Census dataset is also a public dataset that can be downloaded via the SDGym package \footnote{SDGym package website: \url{https://pypi.org/project/sdgym/}}.
This is a clear example of an imbalaned dataset since only 12382 of the samples are considered positive out of a total of 199523 samples.

\textbf{Cervical}

The Cervical cancer dataset comprises demographic information, habits, and historic medical records of 858 patients and was created with the goal to identify the cervical cancer risk factors. The original dataset can be found at the following link: \url{https://archive.ics.uci.edu/ml/datasets/Cervical+cancer+\%28Risk+Factors\%29}.


\textbf{Covtype}

 This dataset contains cartographic variables from four wilderness areas located in the Roosevelt National Forest of northern Colorado and the goal is to predict forest cover type from the  7 possible classes. The data is publicly available at \url{https://archive.ics.uci.edu/ml/datasets/covertype}.

\textbf{Credit}

The Credit Card Fraud Detection dataset contains the categorized information of credit card transactions made by European cardholders during September 2013 and the corresponding label indicating if the transaction was fraudulent or not. The dataset can be found at: \url{https://www.kaggle.com/mlg-ulb/creditcardfraud}. The original dataset has a total number of 284807 samples where only 492 of them are frauds. In our experiments, we descarded the feature related to the time the transaction was done. The data is released under a Database Contents License (DbCL) v1.0.

\textbf{Epileptic}

The Epileptic Seizure Recognition dataset contains the brain activity measured in terms of the EEG across time. The dataset can be found at \url{https://archive.ics.uci.edu/ml/datasets/Epileptic+Seizure+Recognition}. The original dataset contains 5 different labels that we binarized into two: seizure (2300 samples) or not seizure (9200 samples). 

\textbf{Intrusion}

The dataset was used for The Third International Knowledge Discovery and Data Mining Tools Competition held at the Conference  on  Knowledge  Discovery  and  Data  Mining,  1999,  and  can  be  found  at \url{http://kdd.ics.uci.edu/databases/kddcup99/kddcup99.html}. We used the file named "kddcup.data10percent.gz" that contains the $10\%$ of the orginal dataset. The goal is to distinguish between intrusion/attack and normal connections categorized in 5 different labels.


\textbf{Isolet}

The Isolet dataset contains the features sounds as spectral coefficients, contour features, sonorant features, pre-sonorant features, and post-sonorant features of the different letters on the alphabet as inputs and the corresponding letter as the label. The original dataset can be found at \url{https://archive.ics.uci.edu/ml/datasets/isolet}. However, in our experiments we considered this dataset as a binary classification task as we only considered the labels as constants or vowels.

\tabref{data_description} summarizes the number of samples, labeled classes and type of different inputs (numerical, ordinal or categorical) for each tabular dataset used in our experiments. The content of the table reflects the results after pre-processing or binarizing the corresponding datasets.  

\begin{table}[!h]
\caption{Tabular datasets. Size, number of classes and feature types descriptions.}
\label{tab:data_description}
\vskip 0.1in
\centering
\scalebox{0.85}{
\begin{tabular}{l r c c}
\toprule
dataset & $\#$ samps  & $\#$ classes & $\#$ features  \\
\midrule
isolet & 4366 & 2 & 617 num \\
covtype & 406698 &  7 & 10 num, 44 cat \\
epileptic & 11500 & 2 & 178 num \\
credit & 284807 & 2 & 29 num \\
cervical & 753 & 2 & 11 num, 24 cat \\
census & 199523 & 2 & 7 num, 33 cat\\
adult & 48842 & 2 & 6 num, 8 cat\\
intrusion & 394021 & 5 & 8 cat, 6 ord, 26 num\\
\bottomrule\\
\end{tabular}

}
  \vspace{-0.5cm}
\end{table}

\subsection{Hyperparameters for discrete tabular datasets}\label{supp:hyperparams_discrete}

Here we include the hyperparameters used in obtaining the results obtained in Table \ref{tab:alpha_way}. In the main text we describe the choices of the Hermitian hyperparameters. In the separate section \ref{supp:gamma} we present a broader view over the gamma hyperparameter.

 \begin{table}[H]
 \caption{Hyperparameters for discrete tabular datasets}
\label{tab:discete_hyper}
 \centering
 \begin{tabular}{ccccccc}
 \hline
 & \multicolumn{1}{c}{privacy} & \multicolumn{1}{c}{batch rate} & \multicolumn{1}{c}{order hermite prod} & \multicolumn{1}{c}{prod dimension} & \multicolumn{1}{c}{gamma} & 
 \multicolumn{1}{c}{order hermite}\\
 \hline
 \multirow{2}{*}{Adult} & $\varepsilon=0.3$ & 0.1 & 10 & 5 & 1 & 100\\
 & $\varepsilon=0.1$ & 0.1 & 5 & 7 & 1 & 100\\
 \hline
 \multirow{2}{*}{Census} & $\varepsilon=0.3$ & 0.01 & 5 & 7 & 0.1 & 100\\
 & $\varepsilon=0.1$ & 0.01 & 5 & 7 & 0.1 & 100\\

 \hline
 \end{tabular}
 \end{table}

\subsection{Gamma hyperparameter ablation study}
\label{supp:gamma}

Here we study the impact of gamma $\gamma$ hyperparameter on the quality of the generated samples. Gamma describes the weight that is given to the product kernel in relation to the sum kernel. We elaborate on the results from the Table \ref{tab:alpha_way} which describe $\alpha$-way marginals evaluated on generated samples with discretized Census dataset. We fix all the hyperparameters and vary gamma. The Table \ref{tab:gamma} shows the impact of gamma. The $k-$way results remain indifferent for $\gamma \leq 1$ but deterioriate for $\gamma > 1$. In this experiment, we set $\epsilon_1=\epsilon_2=\epsilon/2$.
Here, ``order hermite prod " means the HP order for the outer product kernel, ``prod dimension" means the number of subsampled input dimensions, and ``order hermite" means the HP order for the sum kernel. 

\begin{table}[H]
\caption{The impact of gamma hyperparamer. }
\label{tab:gamma}
\vskip 0.2cm
\centering
\scalebox{0.95}{
\begin{tabular}{cccccccc}
\toprule
epsilon & batch rate & order hermite prod & prod dimension & gamma & epochs & 3-way & 4-way \\ \midrule
0.3 & 0.1 & 10 & 5 & 0.001 & 8 & 0.474 & 0.570 \\
0.3 & 0.1 & 10 & 5 & 0.01 & 8 & 0.473 & 0.570 \\
0.3 & 0.1 & 10 & 5 & 0.1 & 8 & 0.499 & 0.597 \\
0.3 & 0.1 & 10 & 5 & 1  & 8 & 0.474 & 0.570 \\
0.3 & 0.1 & 10 & 5 & 10  & 8 & 0.585 & 0.671 \\
0.3 & 0.1 & 10 & 5 & 100  & 8 & 0.674 & 0.757 \\
0.3 & 0.1 & 10 & 5 & 1000  & 8 & 0.676 & 0.761 \\
\bottomrule
\end{tabular} 
}
\end{table}


\subsection{Training DP-HP generator}\label{supp:tabular_data_DP_HP_hyperparam}

Here we provide the details of the DP-HP training procedure we used on the tabular data experiments. \tabref{data_hyperparam} shows the Hermite polynomial order, the fraction of dataset used in a batch, the number of epochs and the undersampling rate we used during training for each tabular dataset. 

\begin{table}[!ht]
\caption{Tabular datasets. Hyperparameter settings for private constraints $\epsilon=1$ and $\delta=10^{-5}$.}
\label{tab:data_hyperparam}
\vskip 0.1in
\centering
\scalebox{0.85}{
\begin{tabular}{cccccc}
\toprule
data name & batch rate & order hermite prod & prod dimension & order hermite & gamma \\
\midrule
adult & 0.1 & 5 & 5 & 100 & 0.1 \\
census & 0.5 & 5 & 5 & 100 & 0.1 \\
cervical & 0.5 & 13 & 5 & 20 & 1 \\
credit & 0.5 & 7 & 5 & 20 & 1 \\
epileptic & 0.1 & 5 & 7 & 20 & 0.1 \\
isolet & 0.5 & 13 & 5 & 150 & 1 \\ \midrule
covtype & 0.01 & 7 & 2 & 10 & 1 \\
intrusion & 0.01 & 5 & 5 & 7 & 1 \\ 
\bottomrule
\end{tabular} 

}
  \vspace{-0.5cm}
\end{table}

\subsection{Non-private results}

We also show the non-private MERF and HP results in \tabref{non_priv_as_well_summary_tabular}.
\begin{table*}[h]
\caption{Performance comparison on Tabular datasets. The average over five independent runs. 
}
\centering
\scalebox{0.7}{
\begin{tabular}{l *{7}{|cc} }
\toprule
& \multicolumn{2}{ c| }{Real} &  \multicolumn{2}{ c| }{DP-MERF}  &  \multicolumn{2}{ c| }{DP-HP} & \multicolumn{2}{ c| }{DP-CGAN}  &
\multicolumn{2}{ c| }{DP-GAN}  &
\multicolumn{2}{ c| }{{DP-MERF}} &
\multicolumn{2}{ c }{\textbf{DP-HP}}\\ 
&\multicolumn{2}{ c| }{} &  \multicolumn{2}{ c| }{(non-priv)} & \multicolumn{2}{ c| }{(non-priv)} & \multicolumn{2}{ c| }{($1,10^{-5}$)-DP} & \multicolumn{2}{ c| }{($1,10^{-5}$)-DP} & 
\multicolumn{2}{ c| }{($1,10^{-5}$)-DP} &
\multicolumn{2}{ c }{($1,10^{-5}$)-DP}\\ 
\midrule
\textbf{adult} & 0.786 & 0.683 & 0.642 & 0.525 & \textbf{0,673} & \textbf{0,621} & 0.509 & 0.444 & 0.511 & 0.445 & 0.642 & 0.524 & \textbf{0,688 }& \textbf{0,632} \\
\textbf{census} & 0.776 & 0.433 & 0.696 & 0.244 & \textbf{0,707} & \textbf{0,32} & 0.655 & 0.216 & 0.529 & 0.166 & 0.685 & 0.236 & \textbf{0,699} & \textbf{0,328 }\\
\textbf{cervical} & 0.959 & 0.858 & \textbf{0.863 }& \textbf{0.607} & 0,823 & 0,574 & 0.519 & 0.200 & 0.485 & 0.183 & 0.531 & 0.176 & \textbf{0,616 }& \textbf{0,312 }\\
\textbf{credit} & 0.924 & 0.864 & \textbf{0.902} & 0.828 & 0.89 & \textbf{0,863} & 0.664 & 0.356 & 0.435 & 0.150 & 0.751 & 0.622 & \textbf{0,786} & \textbf{0,744} \\
\textbf{epileptic} & 0.808 & 0.636 & 0.564 & 0.236 & \textbf{0,602} & \textbf{0,546} & 0.578 & 0.241 & 0.505 & 0.196 & 0.605 & 0.316 & \textbf{0,609} &\textbf{ 0,554} \\
\textbf{isolet} & 0.895 & 0.741 & 0.755 & 0.461 &\textbf{ 0,789} & \textbf{0,668} & 0.511 & 0.198 & 0.540 & 0.205 & 0.557 & 0.228 & \textbf{0,572} & \textbf{0,498} \\ 

& \multicolumn{2}{ c| }{F1}
& \multicolumn{2}{ c| }{F1}
& \multicolumn{2}{ c| }{F1}
& \multicolumn{2}{ c| }{F1}
& \multicolumn{2}{ c| }{F1}
& \multicolumn{2}{ c| }{F1}
& \multicolumn{2}{ c }{F1}\\ 
\textbf{covtype} & \multicolumn{2}{ c| }{0.820} &  \multicolumn{2}{ c| }{\textbf{0.601}} &
\multicolumn{2}{ c |}{0.580} &
\multicolumn{2}{ c| }{0.285} &
\multicolumn{2}{ c| }{0.492} &
\multicolumn{2}{ c| }{0.467} &  \multicolumn{2}{ c }{\textbf{0.537}}\\
\textbf{intrusion} & \multicolumn{2}{ c| }{0.971} &  \multicolumn{2}{ c| }{0.884} &
\multicolumn{2}{ c| }{\textbf{0.888}} &
\multicolumn{2}{ c| }{0.302} & 
\multicolumn{2}{ c| }{0.251} &
\multicolumn{2}{ c| }{\textbf{0.892}} & \multicolumn{2}{ c }{0.890} \\ 
\bottomrule
\end{tabular}}\label{tab:non_priv_as_well_summary_tabular}
\end{table*} 





\subsection{The effect of subsampled input dimensions for the product kernel on Adult dataset} \label{supp:tradeoff_subsamp_dims}

\tabref{adult_subsampling_prod_dims} shows the 3-way (Left) and 4-way (Right) marginals evaluated at different number of dimensions for the product kernel ($D_{prod}$) where the rest of hyperparameters are fixed. The results show that increasing the number of dimensions in the product kernel improved the result.

\begin{table*}[h]
\caption{Trade-off for subsampling dimensions in the product kernel for Adult dataset.}
\centering
\begin{tabular}{c||ccc|ccc|}
\hline
   &  & $D_{prod}$&  &  & $D_{prod}$ &  \\ 
 $\epsilon$  & $2$ & $5$ & $7$ &  $2$ & $5$ & $7$ \\ \hline
1 & 0.367& 0.34 & \textbf{0.332} & 0.466&  0.434 & \textbf{0.422} \\ 
\hline
\end{tabular} \label{tab:adult_subsampling_prod_dims}
\end{table*}

\section{Image data}\label{supp:image_results}

\subsection{Results by model}

In the following we provide a more detailed description of the downstreams models accuracy over the different methods considered in the image datasets.

\begin{figure}[h!]
 \begin{subfigure}[b]{0.5\textwidth}
 \centering
\includegraphics[width=\textwidth]{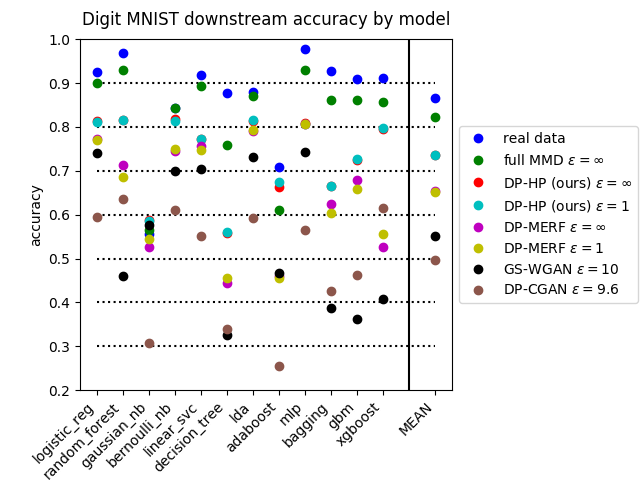}
 \caption{MNIST}
 \label{fig:mnist_by_model_d}
\end{subfigure}
\begin{subfigure}[b]{0.5\textwidth}
 \centering
\includegraphics[width=\textwidth]{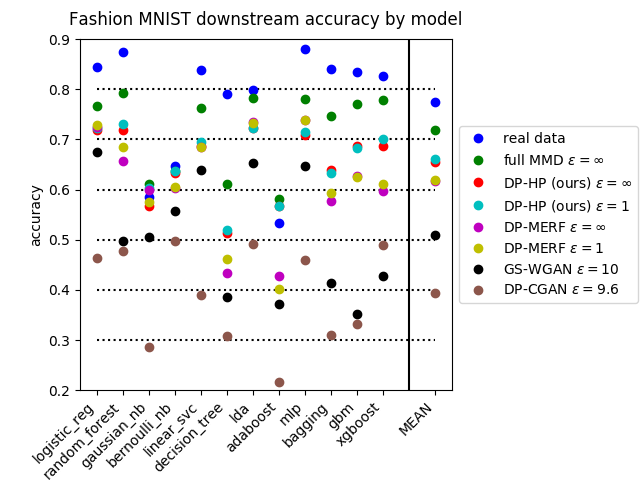}
 \caption{FashionMNIST}
 \label{fig:mnist_by_model_f}
 \end{subfigure}
\caption{We compare the real data test accuracy of models trained on synthetic data for various models: DP-HP, DP-MERF, GS-WGAN and DP-CGAN. As baselines we also include results for real training data and a generator, which is non-privately trained with MMD, listed as "full MMD". We show accuracy sorted by downstream classifier and the mean accuracy across classifiers on the right. Each score is the average of 5 independent runs.}
\label{fig:mnist_by_model}
\end{figure}

\subsection{MNIST and fashionMNIST hyper-parameter settings}\label{supp:hyperparams}

Here we give a detailed hyper-parameter setup and the architectures used for generating synthetic samples via DP-HP for MNIST and FashionMNIST datasets in \tabref{image_hyperparam}. The non-private version of our method does not exhibit a significant accuracy difference between 2, 3 and 4 subsampled dimensions for the product kernel, so we considered product dimension to be 2 for memory savings. \tabref{image_downstream_hyperparam} summarizes the 12 predictive models hyper-parameters setup for the image datasets trained on the generated samples via DP-HP. In this experiment, we optimize this loss $\min_{\vtheta } \; ||\widehat{\vmu}^p_P - \widehat{\vmu}^p_{Q_\vtheta}||_{2}^{2} +  \gamma ||\widehat{\vmu}^s_P - \widehat{\vmu}^s_{Q_\vtheta} ||_{2}^{2}$, where $\gamma$ is multiplied by the sum kernel's loss.  

\begin{table}[!h]
\caption{Hyperparameter settings for image data experiments. All parameters not listed here are used with their default values.}
\label{tab:image_hyperparam}
\vskip 0.1in
\centering
\scalebox{0.80}{
\begin{tabular}{l *{2}{|cc} }
\toprule
& \multicolumn{2}{ c| }{MNIST}  & \multicolumn{2}{ c }{FashionMNIST} \\
& (non-priv) &
($1,10^{-5}$)-DP & (non-priv) &
($1,10^{-5}$)-DP \\
\midrule
Hermite order (sum kernel) & 100 & 100  & 100 & 100\\
Hermite order (product kernel) & 20 & 20  & 20 & 20\\
kernel length (sum kernel) & 0.005  & 0.005 & 0.15 & 0.15 \\
kernel length (product kernel) & 0.005  & 0.005 & 0.15 & 0.15 \\
product dimension & 2 & 2 & 2 & 2\\
subsample product dimension & beginning of each epoch & beginning of each epoch & beginning of each epoch & beginning of each epoch\\
gamma & 5  & 20  & 20 & 10 \\
mini-batch size  & 200 & 200 & 200 & 200 \\
epochs & 10 & 10 & 10 & 10 \\
learning rate & 0.01 & 0.01 & 0.01 & 0.01 \\
architecture & fully connected  & fully connected & CNN + bilinear upsampling & CNN + bilinear upsampling \\
\bottomrule
\end{tabular}
}
  \vspace{-0.5cm}
\end{table} 

\begin{table}[!h]
\caption{Hyperparameter settings for downstream models used in image data experiments. Models are taken from the scikit-learn 0.24.2 and xgboost 0.90 python packages and hyperparameters have been set to achieve reasonable accuracies while limiting runtimes. Paramters not listed are kept at their default values.}
\label{tab:image_downstream_hyperparam}
\vskip 0.1in
\centering
\scalebox{0.85}{
\begin{tabular}{l|l}
\toprule
Model & Parameters \\
\midrule
Logistic Regression & solver: lbfgs, max\_iter: 5000, multi\_class: auto
\\ 
Gaussian Naive Bayes & - \\
Bernoulli Naive Bayes & binarize: 0.5\\
LinearSVC & max\_iter: 10000, tol: 1e-8, loss: hinge\\
Decision Tree & class\_weight: balanced \\
LDA & solver: eigen, n\_components: 9, tol: 1e-8, shrinkage: 0.5\\
Adaboost & n\_estimators: 1000, learning\_rate: 0.7, algorithm: SAMME.R\\
Bagging & max\_samples: 0.1, n\_estimators: 20\\
Random Forest & n\_estimators: 100, class\_weight: balanced\\
Gradient Boosting & subsample: 0.1, n\_estimators: 50\\
MLP & - \\
XGB & colsample\_bytree: 0.1, objective: multi:softprob, n\_estimators: 50\\
\bottomrule
\end{tabular}
}
  \vspace{-0.5cm}
\end{table}



\end{document}